\documentclass{article}

\usepackage{graphicx}
\usepackage{booktabs} % for professional tables
\usepackage{tabularx}
\usepackage{makecell}
\PassOptionsToPackage{hyphens}{url}\usepackage{hyperref}

\usepackage[preprint]{icml2026}
\usepackage{algorithm}
\usepackage{algorithmic}

\usepackage{amsmath}
\usepackage{amssymb}
\usepackage{mathtools}
\usepackage{amsthm}
\usepackage{thmtools} 
\usepackage{thm-restate}
\usepackage{fancybox,framed}
\usepackage{siunitx}
\usepackage{physics}
\usepackage{nicefrac}
\newcommand{\expect}[1]{\mathbb{E}\left[#1\right]}

\usepackage{nameref}
\renewcommand{\algorithmicrequire}{\textbf{Input:}}
\renewcommand{\algorithmicensure}{\textbf{Output:}}
\usepackage[capitalize,nameinlink,noabbrev]{cleveref}
\crefname{assumption}{Assumption}{Assumptions}
\crefname{question}{Question}{Questions}
\crefname{contribution}{Contribution}{Contributions}
\newcommand{\E}{\mathbb{E}}

\newcommand{\Id}{\mathbf{I}_d}

\newcommand{\bstar}{\beta^\star}
\newcommand{\Ball}[2]{B\!\left(#1,#2\right)}

\newcommand{\inner}[2]{\langle #1, #2 \rangle}
\theoremstyle{plain}
\newtheorem{theorem}{Theorem}[section]

\newtheorem{lemma}{Lemma}[section]

\newtheorem{corollary}{Corollary}[section]
\theoremstyle{definition}
\newtheorem{definition}{Definition}[section]
\newtheorem{contribution}{Contribution}
\newtheorem{question}{Question}
\newtheorem{assumption}{Assumption}[section]
\theoremstyle{remark}
\newtheorem{remark}{Remark}[section]



\usepackage[table]{xcolor}
\usepackage{array}
\usepackage{amssymb} % For nicer checkmarks/crosses if needed, or pifont

\usepackage{pifont}
\newcommand{\cmark}{\ding{51}} % Check
\newcommand{\xmark}{\ding{55}} % Cross
\newcommand{\pmark}{\ding{109}}

\usepackage{setspace}
\icmltitlerunning{Robust SIMs beyond Monotonicity}

\begin{document}

\onecolumn

  % \icmltitle{Robust Single Index Models beyond Monotonicity\\ under Strong Adversarial Corruption}
  \icmltitle{Convex Basins in Single-Index Model Loss Landscapes:\\Applications to Robust Recovery under Strong Adversarial Corruption}


  \icmlsetsymbol{equal}{*}

  \begin{icmlauthorlist}
    \icmlauthor{Santanu Das}{yyy,equal}
    \icmlauthor{Sagnik Chatterjee}{yyy,equal}
    \icmlauthor{Jatin Batra}{yyy}
    
  \end{icmlauthorlist}

  \icmlaffiliation{yyy}{School of Technology and Computer Science, Tata Institute of Fundamental Research, Mumbai, India}

  \icmlcorrespondingauthor{Santanu Das}{dassantanu315@gmail.com}
  % \icmlcorrespondingauthor{Sagnik Chatterjee}{chatsagnik@gmail.com}

  \icmlkeywords{Robust Estimation, Single Index Models, Heavy Tailed Noise, Phase Retrieval}

  \vskip 0.3in



\printAffiliationsAndNotice{\icmlEqualContribution}

\begin{abstract}
  We study the problem of robustly learning Gaussian Single Index Models (SIMs) in the presence of heavy-tailed noise and a constant fraction of adversarially corrupted covariates and responses. Prior work on robust recovery has considered settings such as linear regression (Pensia et al., JASA 2024), strictly monotonic link functions (Awasthi et al., NeurIPS 2022), and phase retrieval (Buna and Rebeschini, AISTATS 2025). However, these techniques do not extend to generic asymmetric non-monotonic link functions such as \textsc{GeLU} and \textsc{Swish}, which arise naturally as scalar primitives in modern gated neural architectures. We close this gap by giving the first robust recovery algorithm with near-linear sample and time complexity for generic non-monotonic link functions, thereby establishing the first robust recovery guarantees for a broad family of nonlinear SIMs for which \textit{no guarantees were previously known}. Our central contribution is a new structural understanding of the Gaussian squared-loss landscape under adversarial contamination. Crucially, we prove that for a broad class of nonlinear non-monotonic SIMs, a dimension-independent, constant-radius convex basin exists around the ground truth and is efficiently reachable via robust spectral initialization even under adversarial contamination. Prior works fail to establish both guarantees simultaneously, thereby either breaking down under adversarial contamination or failing to handle generic non-monotonic link functions. Together, these structural insights yield a principled warm start for robust gradient descent that provably converges to a final estimation error of $O(\sigma\sqrt{\epsilon})$ in $\tilde{O}(nd)$ time with $\tilde{O}(d)$ samples, where $\epsilon$ is the contamination fraction.
\end{abstract}

\section{Introduction}
Single-Index Models (SIMs) are a broad family of semi-parametric models subsuming linear regression, logistic regression, phase retrieval, and generalized linear models as special cases. They model the response variable $Y\in\mathbb{R}$ as a nonlinear function of a one-dimensional projection of the covariates $X\in\mathbb{R}^d$:
\begin{equation}\label{eq:SIM}
    Y = f(X^\top\beta^\star) + \zeta,
\end{equation}
where $f:\mathbb{R}\to\mathbb{R}$ is a known \textit{link} function, $\zeta$ is stochastic noise, and $\beta^\star\in\mathbb{R}^d$ is an unknown index vector to be recovered. The recovery of $\beta^\star$ is a fundamental problem in semi-parametric statistics~\cite{bc64,MC84,ICH93semiparametric,
CFGW997,HJS2001direct,dalalyan2008new} and machine learning~\cite{bruna2025survey}. In this paper, we focus on Gaussian designs $X\sim\mathcal{N}(0,\mathbf{I}_d)$, the canonical setting for studying the geometry of non-convex SIM loss landscapes~\cite{bruna2025survey,barbier2019optimal, 
mondelli2018fundamental,lu2020phase,arous2021online,damian2024smoothing,
damian2024computational,srebro25}.

While classical SIM recovery techniques relied on clean-data assumptions~\citep{hardle1989investigating,li1991sliced}, 
in practice, data is invariably subject to noise and corruption. 
The field of robust statistics~\citep{huber1992robust,tukey1975mathematics,hampel2011robust} developed estimators with well-understood robustness guarantees such as high breakdown points, primarily in low-dimensional settings. However, in high dimensions, achieving strong robustness often leads to estimators that are computationally intractable, with many classical formulations known to be NP-hard~\citep{johnson1978densest,bernholt2006robust}.
Recent breakthroughs in algorithmic robust statistics~\citep{diakonikolas2019robust,diakonikolas2019efficient} overcame this computational barrier by 
providing efficient high-dimensional subroutines for robust mean estimation and robust PCA that tolerate both 
heavy-tailed noise and strong adversarial corruption. These subroutines serve as algorithmic building blocks 
for efficient robust recovery in structured models such as linear 
regression~\citep{cherapanamjeri2020optimalrobustlinearregression,prasad2020robust}, logistic 
regression~\citep{Sever,pensiastream}, and phase 
retrieval~\citep{dong2025outlier,bunaR24_robust_phase,das2026}.

For monotonic link functions, \citet{kalai2009isotron} gave the first efficient recovery algorithm for monotone Lipschitz SIMs in the clean setting via the Isotron algorithm, later extended to broader classes of SIMs by \citet{kakade2011efficient} and 
\citet{plan}. Under heavy-tailed noise and strong adversarial contamination, near-linear time algorithms with optimal sample complexity $\tilde{O}(d)$ were obtained for linear regression, i.e., $f(x)=x$, ~\citep{cherapanamjeri2020optimalrobustlinearregression,
prasad2020robust} with \citet{Pensia} later obtaining information-theoretically optimal error rates under the same guarantees. Beyond linear regression, monotonic link functions such as logistic regression have been studied in the context of Generalized Linear 
Models (GLMs)~\citep{awasthi2022trimmed}; for robust recovery under heavy-tailed noise and strong adversarial contamination, \citet{Sever} proposed a polynomial-time, 
optimal-sample-complexity algorithm, while \citet{pensiastream} later obtained near-linear time at the cost of polynomial sample complexity in the streaming setup. For non-monotonic link functions, phase retrieval, i.e., $f(z) = z^2$, stands out as the canonical example, with applications in optics, crystallography, X-ray imaging, and astrophysics. \citet{candes2015phase} and \citet{NetrapalliJS15} gave the first efficient recovery algorithms under Gaussian covariates in the absence of noise and corruption. Under strong adversarial contamination without noise ($\zeta = 0$), \citet{dong2025outlier} proposed a near-linear time algorithm with optimal sample complexity $\tilde{O}(d)$. \citet{bunaR24_robust_phase} extended these results to heavy-tailed noise with an exponential-time algorithm, which was recently subsumed by \citet{das2026}, who gave a polynomial-time optimal sample complexity robust recovery algorithm for the same setting.

We note that the lack of robust recovery results for generic non-monotonic SIMs beyond phase retrieval is telling: standard first-order proof techniques ~\citep{arous2021online,ren2025emergence,arous2025learning} break down under strong adversarial contamination 
in high dimensions, since the adversary destroys the statistical structure that gradient-based convergence arguments rely upon. The special case of phase retrieval enjoys two structural properties that are conducive for robust recovery: (i) its squared-loss landscape admits a convex basin of constant radius around the true parameter, enabling second-order convergence guarantees within the basin, and (ii) its non-vanishing second Hermite coefficient\footnote{The $k$-th Hermite coefficient of $f$ is the coefficient of the $k$-th Hermite polynomial in the Gaussian expansion of $f$.} ensures the signal direction is reachable via spectral methods ~\citep{candes2015phase,NetrapalliJS15}. It is a priori unclear whether these properties extend to generic non-monotonic SIMs, which often possess a high information exponent\footnote{The information exponent (IE) of a link function is the order of its first non-zero Hermite coefficient; a high IE means the signal is suppressed in low-order moments.} (IE). Even in the clean setting, assuming one existed, accessing a convex basin 
for generic SIMs would require computationally demanding 
methods such as Tensor PCA~\citep{anandkumar2017homotopy}, in stark contrast to the simple spectral methods that suffice for phase retrieval. Indeed, it remains unclear for which classes of non-monotonic link functions efficient provable robust recovery is achievable. This uncertainty naturally brings us to the following question:
\begin{question}
    Can we characterize the class of link functions which admit \textit{efficient} 
    provable robust recovery guarantees under heavy-tailed noise and strong adversarial 
    contamination?
\end{question}

\textbf{Our Contributions:} 
We make significant progress on the above question by identifying two general structural conditions on the link function $f$ outlined in \cref{assump:con-basin,assump:sq-curv}, under which efficient provable robust recovery is achievable. These conditions are remarkably general, capturing a large class of SIMs with generative 
exponent\footnote{The generative exponent (GE) of a link function is the minimum IE achievable over all square-integrable label transforms of $f$~\citep{damian2024computational}. For SIMs having GE at most 2, the signal is detectable via second-order spectral methods after an appropriate label transform.} at most $2$. This class includes phase retrieval, \textsc{Tanh}, \textsc{Probit}, and \textsc{Logistic}, as well as modern activation functions such as \textsc{GeLU}~\citep{gelu}, \textsc{Swish}~\citep{swish} which arise naturally as scalar primitives in gated neural architectures~\cite{shazeer2020gluvariantsimprovetransformer} that serve as the fundamental building blocks of Transformer architectures (e.g., GPT and BERT~\cite{gpt,bert}). We now state an informal characterization of our main result 
below.

\begin{theorem}[Linear Sample and Time Robust Recovery, Informal]
    Consider a SIM with a link function satisfying ~\cref{assump:con-basin,assump:sq-curv}. There exists an algorithm that, using $n=\tilde{O}(d)$ samples and tolerating a constant fraction $\epsilon$ of adversarial contamination under heavy-tailed noise with variance $\sigma^2$, outputs an estimate $\hat{\beta}$ s.t. $\|\hat{\beta} - \beta^\star\|_2 = O(\sigma\sqrt{\epsilon})$ in $\tilde{O}(n d)$ time, with high probability.
\end{theorem}
\subsection{Technical Overview}
A natural strategy for robust recovery in generic non-monotonic SIMs is to leverage the geometry of the squared-loss landscape. Global convexity of the loss landscape would guarantee that any optimizer avoids spurious local minima, but for generic non-monotonic SIMs, global convexity of the squared loss holds if and only if $f$ is affine, reducing to linear regression. For non-affine link functions, a straightforward approach is therefore to establish the existence of a convex neighborhood around the 
true parameter $\beta^\star$. Once an optimizer enters this region, standard convex optimization guarantees ensure reliable recovery. However, for this strategy to be computationally viable in high dimensions, the convex basin must have radius $R = O(1)$ 
strictly independent of the ambient dimension $d$, since a vanishing basin of attraction offers no algorithmic guarantee in high-dimensional non-convex optimization. Prior to this work, such a dimension-independent convex basin was known to exist only for phase retrieval and monotonic link functions. This motivates our first contribution.

\begin{contribution}\label{contrib1}
We identify a sufficient condition (see \cref{assump:con-basin}) under which the squared-loss landscape admits a convex basin of dimension-independent, constant radius $R = O(1)$ around the true parameter $\beta^\star$, for a wide class of link functions 
including \textsc{GeLU}, \textsc{Swish}, \textsc{Tanh}, \textsc{Probit}, \textsc{Logistic}, and phase retrieval. This is the first such guarantee for any non-affine, non-monotonic 
link function beyond phase retrieval.
\end{contribution}

Establishing the existence of a convex basin does not by itself yield a computational guarantee. One must also certify that the basin can be reached efficiently from an initialization, even under heavy-tailed noise and strong adversarial contamination. 
The recent works of \citet{bunaR24_robust_phase} and \citet{das2026} make progress in this direction for phase retrieval, but their approaches suffer from two limitations. First, the robust PCA subroutines they employ run in time polynomial in the dimension. Second, and more fundamentally, their reachability arguments exploit the symmetric structure of the quadratic link and do not extend to asymmetric non-monotonic link functions such as \textsc{GeLU} and \textsc{Swish}. Beyond these computational limitations, certifying reachability under strong adversarial contamination in high dimensions poses a deeper analytical challenge. Prior proof techniques~\citep{arous2021online,ren2025emergence,arous2025learning} that implicitly leverage convex basin structure rely on martingale-drift decompositions that require the stochastic deviations to be mean-zero, a property the adversary destroys by corrupting a constant fraction of samples in high dimensions. Our second contribution resolves both the computational and analytical obstacles.
 
\begin{contribution}\label{contrib2}
We identify a sufficient condition (see \cref{assump:sq-curv}), termed \textit{Expected Squared Convexity (ESC)}, which characterizes when the leading eigenvector of (higher-order) moment estimators aligns with the signal $\beta^\star$. Further, under the ESC condition, for link functions whose squared-loss landscape admits a constant-radius convex basin around $\beta^\star$, off-the-shelf robust spectral initialization on these (higher-order) sample moment matrices yields an estimate $\beta_0$ that (i) lies in the convex basin under heavy-tailed noise and strong adversarial contamination, and (ii) estimates the true parameter $\beta^\star$ with additive error $O(\epsilon^{1/4})$. This is the first explicit, efficiently checkable guarantee that a constant-radius convex basin is reachable for generic non-monotonic SIMs under adversarial contamination.
\end{contribution}

\cref{contrib2} demonstrates that the shortcomings of the previous approaches~\citep{bunaR24_robust_phase,das2026} are not intrinsic. By coupling second-order Stein identities with a refined analysis of robust spectral estimation, we generalize the reachability insights of phase retrieval to a broad class of non-monotonic link functions while achieving near-linear initialization time. Under the Gaussian design, second-order Stein's identity allows us to decompose the 
population moment matrix into a rank-one component aligned with $\beta^\star$ and an isotropic component. Crucially, we show that the coefficient of the signal direction in this decomposition is governed precisely by the ESC condition, which allows us to identify the moment matrix whose leading eigenvector is $\beta^\star$. Using this structural insight, we apply the off-the-shelf robust PCA subroutine of ~\citet{jambulapati2024black}, which explicitly leverages the hypercontractivity of the Gaussian design to extract the leading eigenvector of the sample moment matrices 
efficiently under adversarial contamination. This contrasts with earlier analyses ~\citep{bunaR24_robust_phase,das2026,yang2017}, which either fell short of this general structural characterization or did not exploit such concentration properties to ensure both robustness and efficiency.

However, attaining better error rates requires going beyond spectral initialization. In classical phase retrieval algorithms like Wirtinger Flow~\cite{candes2015phase}, spectral methods are used primarily to initialize the optimization within a basin of attraction, after which Gradient Descent (GD) is employed to converge to the exact solution. In the robust setting, this two-stage architecture is equally critical but for a fundamental statistical reason: we observe that relying solely on robust spectral estimators encounters a fundamental error floor of $O(\epsilon^{1/4})$. To break this barrier, we use an off-the-shelf {robust Gradient Descent (GD)} subroutine that refines the initial estimate significantly while maintaining near-linear time complexity.

\begin{contribution} \label{contrib3}
We provide the \textit{first} near-linear time, \textit{optimal} sample complexity algorithm for robust recovery of a wide class of link functions under heavy-tailed noise and strong adversarial contamination. By initializing via higher-order robust spectral methods and optimizing with robust GD, we achieve a significantly better additive estimation error of $O(\sqrt{\epsilon})$. Notably, this constitutes the first efficient robust recovery guarantee for any non-monotonic link function beyond phase retrieval, in particular for widely used activations such as GeLU and Swish.
\end{contribution}
\subsubsection{Related Work}
To put our contributions in context, prior to this work, it was not known if \textit{efficient} robust recovery in under both heavy-tailed noise and strong adversarial contamination was possible for generic link functions. We now give a brief overview of known results for robust recovery in SIMs beyond linear regression, and defer a detailed literature review to the appendix (see~\cref{sec:litreview}).

 \citet{Sever} proposed a polynomial-time $\tilde{O}(d)$ sample recovery algorithm (SEVER) for logistic regression under strong adversarial contamination with respect to the hinge loss and the logistic loss. With respect to the squared loss, their sample complexity blows up to $\tilde{O}(d^5)$. \citet{pensiastream} obtain a near-linear time robust recovery streaming algorithm for logistic regression that requires $\tilde{O}(d^2)$ samples, with respect to the squared loss. We note that both \citet{Sever} and \citet{pensiastream} obtain an error of $\sigma\sqrt{\epsilon}$, similar to us.
\citet{awasthi2022trimmed} obtain optimal sample complexity robust recovery guarantees under strong adversarial contamination for GLMs (which rely on only monotonic link functions) \textit{without any guarantees} on the running time of their approach, and only under Gaussian noise, in contrast to our near-linear time robust recovery algorithm under heavy-tailed noise and strong adversarial contamination. However, our error rate of $O(\sigma\sqrt{\epsilon})$ is worse than their $O(\sigma\epsilon\log\frac{1}{\epsilon})$ error rate, which is \textit{optimal}. The caveat, however, is that our error rate also holds for a large class of non-monotonic link functions, which are not addressed by \citet{awasthi2022trimmed}. We also remark that our error rate matches the best known error rate for existing robust recovery algorithms for non-monotonic link functions~\cite{Sever,pensiastream,bunaR24_robust_phase,das2026} under heavy-tailed noise and strong adversarial contamination. 

\subsection{Organization of the paper} We detail our problem setup, model, and technical tools in \cref{sec:prelims}. In \cref{sec:con-basin}, we prove the existence of a convex basin. Next, we obtain a linear-time optimal-sample complexity robust recovery algorithm in \cref{sec:lin}. Finally in \cref{sec:dis}, we discuss our contributions and outline multiple avenues for future directions.

\section{Preliminaries}\label{sec:prelims}

\textbf{Notation} For any positive integer $n$, let $[n]$ denote the set $\{1, 2, \ldots, n\}$.
For a vector $v \in \mathbb{R}^d$, $\|v\|_2$ denotes its Euclidean norm. 
For two vectors $u, v \in \mathbb{R}^d$, $\langle u, v \rangle$ denotes their inner product.
For a matrix $M$, $\|M\|_{\mathrm{op}}$ denotes its operator norm, $\operatorname{Tr}(M)$ its trace, and we write $M \succeq 0$ to indicate that $M$ is PSD. For a symmetric matrix $M$, $\lambda_{\min}(M)$ and $\lambda_{\max}(M)$ denote its smallest and largest eigenvalues, respectively.
We denote the identity matrix in $d$ dimensions by $\mathbf{I}_d$. 
The indicator function of an event $E$ is denoted by $\mathbb{I}(E)$.
$\mathcal{N}(\mu, \Sigma)$ denotes a multivariate Gaussian with mean $\mu$ and covariance $\Sigma$. 
$\mathbb{E}[\cdot]$ denotes expectation. 
$\widetilde{O}(\cdot)$ hides logarithmic factors in the dimension $d$, sample size $n$, corruption level $\epsilon$, and failure probability $\delta$. 
Finally, $f^{(1)} := f'$, and $f^{(2)} := f''$ denote the first and second derivatives of the link function $f$, respectively.

\textbf{Problem Setup}
We consider the single index model (SIM) defined in~\eqref{eq:SIM}, with covariates $X \in \mathbb{R}^d$ and a known\footnote{In non-parametric setups, $f$ is assumed unknown~\cite{zhu06}, 
whereas in machine learning it is either known or constrained to a well-behaved 
class~\cite{kalai2009isotron}.} link function $f$.
In this section, we formally introduce the problem our paper tackles. We first define the notion of strong adversarial contamination, and then introduce our model.
\begin{definition}[Strong Adversarial Contamination~\cite{diakonikolas2019robust}]\label{def:strongcontamination}
    Given a contamination tolerance level $\epsilon\in\left(0,\nicefrac{1}{2}\right)$ and a distribution $\mathcal{D}$ on $\mathbb{R}^d$, generate a clean training set by drawing $n$ samples from $\mathcal{D}$. The adversary is allowed to inspect the entire training set and alter up to $\epsilon n$ of the clean samples arbitrarily. This modified training set of $n$ points is then provided as input to the algorithm. We refer to the modified training set as $\epsilon$-strongly contaminated dataset. 
\end{definition}
% We now formally define our model: 
% SIMs with Gaussian design subject to zero-mean Value Theorem, independent, heavy-tailed, homoscedastic noise subject to strong adversarial contamination.
\begin{definition}[The Model]\label{def:model}
A dataset of the form $\{( x_i, y_{i})\}_{i=1}^n$ is generated such that the responses $y_i$ are drawn from an unknown parameter $\beta^\star \in \mathbb{R}^d$ according to~\cref{eq:SIM}, 
with covariates $x_i \sim \mathcal{N}(0, \mathbf{I}_d)$, a link function $f$ satisfying \cref{assump:con-basin,assump:sq-curv}, and $\|\beta^\star\|_{2}=1$.\footnote{Although this is standard practice \cite{brunaSIMbeyondgaussian, yang2017,dong2025outlier}, we leave the removal of this assumption to future work.} The noise variables $\zeta_i$ are heavy-tailed, zero-mean conditioned on $x_i$, and homoscedastic with bounded variance and fourth moments, i.e., $\mathbb{E}[\zeta_i \mid x_i] = 0$, $\mathbb{E}[\zeta_i^2 \mid x_i] = \sigma^2$, and $\mathbb{E}[\zeta_i^4 \mid x_i] = K_4^4$. An $\epsilon$ fraction of the dataset $\{(y_i, x_i)\}_{i=1}^n$ is then corrupted by a strong adversary, as in~\cref{def:strongcontamination}. 
\end{definition}

\begin{definition}[The Robust Recovery Problem for SIMs]
    Given access to an $\epsilon$-strongly contaminated training set $\{(y_i, x_i)\}_{i=1}^n$ (see in~\cref{def:model}), output w.h.p., a unit norm estimate $\beta \in \mathbb{R}^d$ s.t.
$
\|\beta - \beta^\star\|_2 =O(\sigma\sqrt{\epsilon}).
$
\end{definition}

We now state the natural structural assumptions we require on the link function. 
The first assumption controls the smoothness of the population loss landscape through one-dimensional expectations of $f$ and its derivatives under the Gaussian measure. First, we establish some notation. We consider the population loss $\mathcal{L}(\beta):= \frac{1}{2}\mathbb{E}[(f(X^{\top}\beta)-Y)^{2}]$ with the corresponding Hessian $H(\beta):=\nabla^2\mathcal{L}(\beta)=$ $$\mathbb{E}[\left(f'(X^{\top}\beta))^{2} + (f(X^{\top}\beta) - Y) f''(X^{\top}\beta) \right) X X^{\top}].$$
Let $Z:=X^\top\beta\sim\mathcal{N}(0,\|\beta\|_2^2)$. For $R>0$ we define the local Gaussian moment envelopes
\[
    M_k(R) := \sup_{\|\beta-\beta^\star\|\leq R} 
    {\mathbb{E}}
    \left[\left(f^{(k)}(Z)\right)^4\right]^{\frac{1}{4}},  k \in [2],
\]

\begin{assumption}\label{assump:con-basin}
Given a twice continuously differentiable link function $f:\mathbb{R}\to\mathbb{R}$, 
\begin{itemize}
    \item $\mu:=\lambda_{\operatorname{min}}(\nabla^2\mathcal{L}(\beta^*))>0$,
    \item There exists a constant radius $R>0$ s.t. $M_1(R)<\infty$, $M_2(R)<\infty$, and
${ (6\sqrt{15}\ C_{\mathrm{lip}}(R))}\ R \leq {\mu}$, where $$C_{\mathrm{lip}}(R):=M_1(R)M_2(R).$$ 
\end{itemize}
\end{assumption}
In \cref{assump:con-basin}, $C_{\mathrm{lip}(R)}$ is a local complexity measure that measures how quickly the curvature of the population loss can deteriorate as $\beta$ moves away from $\beta^\star$ inside the ball $\mathcal{B}(\beta^\star,R)$. Crucially, since $M_k(R)$, $k\in[2]$ are evaluated pointwise on $Z \sim \mathcal{N}(0, \|\beta\|_2^2)$, the quantities $M_1, M_1$ (and hence $C_{\mathrm{lip}}(R)$) are 
determined entirely by one-dimensional Gaussian integrals of $f$ and its derivatives. Therefore, any link function satisfying \cref{assump:con-basin} has a dimension-independent basin radius $R$. 

\begin{remark}
    We note that \cref{assump:con-basin} only imposes a mild regularity condition on the link function $f$. We only require  $f', f''$ to have uniformly bounded fourth moments under the Gaussian measure.
\end{remark}

The second assumption complements the first by addressing a different question: while \cref{assump:con-basin} guarantees that the loss landscape admits a dimension-independent convex basin near $\beta^\star$, it does not guarantee (efficient) reachability to this convex basin from a random initialization. \cref{assump:sq-curv}, termed Expected Squared Convexity (ESC), precisely characterizes when the signal direction $\beta^\star$ is identifiable from the second-order moments of the data, enabling a robust spectral initialization that lands within the convex basin.

\begin{assumption}\label{assump:sq-curv}
Let $f$ be a twice differentiable link function, and let $X \sim \mathcal{N}(0, \mathbf{I}_d)$. For any $\beta \in \mathbb{R}^d$, the expected squared convexity of the link function 
$f : \mathbb{R} \to \mathbb{R}$ at $\beta$ is defined as
$$
\mathrm{ESC}(\beta,f):=\mathbb{E}\!\left[
\bigl(f'(X^\top \beta)\bigr)^2
+ f(X^\top \beta)\, f''(X^\top \beta)
\right].
$$
A twice differentiable link function $f$ is strictly ESC if  
$\mathrm{ESC}(\beta^\star,f)>0$.
\end{assumption}

\begin{remark}
    To see why ESC is the right condition for identifiability in non-monotonic links, note that by the product rule, $\mathrm{ESC}(\beta, f) = \mathbb{E}\!\left[\left(f^2(X^\top\beta)\right)''\right]$. Hence, ESC is a higher-order analogue of monotonicity: just as $\mathbb{E}[f'(Z)] > 0$ ensures that the first-order moments of the data carry information about $\beta^\star$, $\mathrm{ESC}(\beta^\star, f) > 0$ ensures that the second-order moments of the data carry information about $\beta^\star$.
\end{remark}

\subsection{Tools}
In this subsection, we present the main tools used in our work.
 \begin{restatable}[Second-order Multivariate Stein's Lemma]{lemma}{steins}\label{lem:steins}
    Let $X\sim\mathcal{N}(0,\mathbf{I}_d)$. For any twice-differentiable function $g:\mathbb{R}^d\mapsto\mathbb{R}$ s.t. $\expect{\nabla^2 g(X)}$ exists, we have
    $
    \expect{g(X)(XX^T-\mathbf{I}_d)}=\expect{\nabla_X^2\, g(X)}.
    $
\end{restatable}
The univariate version of \cref{lem:steins} states that for a Gaussian variable $z \sim \mathcal{N}(0,1)$, $\mathbb{E}[g(z)(z^2 - 1)] = \mathbb{E}[g''(z)]$. 
We now present the main results from algorithmic robust statistics that our paper builds upon.
\begin{lemma}[Robust Mean Estimation(Theorem 3.2 of~\citet{pensiastream})]\label{thm:lin-tim-robut-Mean Value Theorem}
Let $\mathcal{D}$ be a distribution on $\mathbb{R}^d$ with unknown mean $\mu$ and covariance $\Sigma$ satisfying $
\Sigma \preceq \sigma'^2 I$, for some constant $\sigma'>0$. 
For $\epsilon$ smaller than a sufficiently small universal 
constant and $\delta > 0$, given an $\epsilon$-corrupted dataset (see~\cref{def:strongcontamination}) of $n = \tilde{O}(d/\epsilon)$ 
samples, there exists an algorithm running in time 
$\tilde{O}(nd\operatorname{polylog}(d,n,1/\epsilon,1/\delta))$ that outputs $\widehat{\mu}$ 
satisfying, w.p. $\geq 1-\delta$, $\|\widehat{\mu}-\mu\|_2 = O\!\left(\sigma'\sqrt{\epsilon}\right).$
\end{lemma}

We use a robust PCA subroutine of~\cite{jambulapati2024black}, for which we recall 
two definitions. For $M \in \mathbb{S}^{d \times d}_{\succeq 0}$ and $\rho \in [0,1]$, 
a unit vector $u \in \mathbb{R}^d$ is a $\rho$-approximate energy $1$-PCA of $M$ 
\citep[Definition 2]{jambulapati2024black} if
$\langle u u^\top, M \rangle \ge (1-\rho)\,\|M\|_1$,
where $\|M\|_1 := \max_{\|v\|_2=1} \langle vv^\top, M\rangle$.
A random vector $X \in \mathbb{R}^d$ is $(p, C_p)$-hypercontractive 
\citep[Definition 8]{jambulapati2024black} if for all $u \in \mathbb{R}^d$,
$\mathbb{E}[\langle u, X\rangle^p]^{1/p} \le C_p(\mathbb{E}[\langle u, X\rangle^2])^{1/2}$,
for some constant $C_p$ and even integer $p$. We now state the linear-time robust 
PCA guarantee \citep[Theorem 4]{jambulapati2024black}.

\begin{lemma}[Robust hypercontractive $1$-ePCA]
\label{thm:robust_hypercontractive_kePCA}
Let $\mathcal D$ be a $(4,C_4)$-hypercontractive distribution on $\mathbb R^d$ with second moment matrix being $\Sigma$.\footnote{\citet{jambulapati2024black} provide robust PCA guarantees for covariance matrices; these results apply equally to 
second-order moment matrices.
}
Let $\epsilon \in (0,\epsilon_0)$, $\delta \in (0,1)$, and
$\rho = \Theta(C_4^2\sqrt{\epsilon}) \in (0, \rho_0)$ for absolute constants 
$\epsilon_0, \rho_0 > 0$.
Given an $\epsilon$-corrupted dataset $T$ of size
$|T| = \Theta\!\left(\vartheta \cdot \frac{d\log d + \log(1/\delta)}{\rho^2}\right)$,
where $\vartheta := C_4^6/\sqrt{\epsilon}$, there exists an algorithm $\mathcal{A}_k$,\footnote{$\mathcal A_k$ refers to Algorithm 1 of \cite{jambulapati2024black} with Algorithm 2 of 
\cite{robust_PCA_Algorithm_in_nearly_linear_time_by_Diakonikolas} as the $1$-ePCA 
oracle.}
that takes $T, \epsilon, \rho, \delta$ as input and outputs a unit vector 
$\hat{u} \in \mathbb{R}^d$ in time
$O\!\left(\frac{nd}{\rho^2}\operatorname{polylog}\!\left(\frac{d}{\epsilon\delta}\right)\right)$
s.t., w.p. $\geq 1-\delta$, $\hat{u}$ is an $O(\rho)$-approximate energy 
$1$-PCA of $\Sigma$.
% that takes as inputs
% $T, \epsilon, \gamma$ and $ \delta$, outputs a unit vector
% $\hat{u} \in \mathbb R^{d \times 1}$ in time $
% O\!\left(
% \frac{n d}{\gamma^2}
% \operatorname{polylog}\!\left(\frac{d}{\epsilon \delta}\right)
% \right)
% $ s.t., w.p. $\geq 1-\delta$,
% $\hat{u}$ is an $O(\gamma)$-approximate energy $1$-PCA of $\Sigma$.
\end{lemma}

\section{Existence of Convex Basins in the Loss Landscape of SIMs}\label{sec:con-basin}
We now give a sufficient condition to characterize a class of link functions that admit a dimension-independent constant-sized convex basin around $\beta^\star$ in the loss landscape, in the following theorem:
\begin{restatable}{theorem}{lemConvexBasin}
\label{lem:local_geometry}
    Let $Z \sim \mathcal{N}(0,1)$. Define the second and fourth moment proxies, $\mu = \min\left\{\mathbb{E}\left[f'(Z)^2\right],~ \mathbb{E}\left[Z^2f'(Z)^2\right]\right\}$, and $\mu_{1}= \max\left\{\mathbb{E}\left[f'(Z)^2\right],~ \mathbb{E}\left[Z^2f'(Z)^2\right]\right\}$.
    Consider the model as given in \cref{def:model}, s.t. the link function $f$ satisfies \cref{assump:con-basin} with radius $R>0$, 
    then for all $\beta$ in the Euclidean ball $\mathcal{B}(\beta^\star, R)$, the Hessian $H(\beta) := \nabla^{2}\mathcal{L}(\beta)$ satisfies
    $$\frac{\mu}{2}\; \mathbf{I}_{d} \preceq H(\beta) \preceq \left(\frac{\mu}{2} + \mu_{1}\right) \mathbf{I}_{d}.$$
\end{restatable}

\cref{lem:local_geometry} establishes that whenever the link function satisfies
\cref{assump:con-basin}, the population loss landscape is $\frac{\mu}{2}$-strongly
convex and $\frac{\mu + 2\mu_{1}}{2}$-smooth throughout $\mathcal{B}(\beta^\star,R)$.
Crucially, both the strong convexity constant $\mu/2$ and the basin radius $R$ are
\emph{dimension-independent}: they are determined entirely by one-dimensional integrals
of $f$ and its derivatives against the standard Gaussian.  This dimension-independence
is the key structural fact that enables our robust recovery guarantees in
\cref{sec:lin}, since it allows the curvature of the loss to dominate the
adversarial bias uniformly over the entire ball $\mathcal{B}(\beta^\star,R)$, regardless
of the ambient dimension $d$.
We present an extended proof sketch below and defer the full proof to
\cref{sec:proofscon-basin} in the appendix.
\begin{proof}[Proof Sketch of \cref{lem:local_geometry}]
The proof proceeds in three stages: 
(i) characterizing the exact spectrum of the population Hessian at $\beta^\star$;
(ii) controlling how much the Hessian changes as $\beta$ moves away from $\beta^\star$; and
(iii) combining these bounds to establish strong convexity and smoothness throughout 
$\mathcal{B}(\beta^\star,R)$.

\textbf{Hessian Decomposition.}
For any $\beta$ near $\beta^\star$, write
\[
H(\beta)
=
H(\beta^\star)+\Delta(\beta),
\qquad
\Delta(\beta)
:=
H(\beta)-H(\beta^\star).
\]
By Weyl's inequality, $\lambda_{\min}(H(\beta))
\geq
\lambda_{\min}(H(\beta^\star))
-
\norm{\Delta(\beta)}_{\mathrm{op}}.$
Thus, it is enough to show that $H(\beta^\star)$ has dimension-independent positive curvature and that the perturbation remains sufficiently small throughout the ball.

\textbf{Strong Convexity at the Optimum.}
At $\beta^\star$, $H(\beta^\star)
=
\E\!\left[
f'(X^\top\beta^\star)^2 XX^\top
\right].$ 
Using $\norm{\beta^\star}_2=1$ and Gaussian symmetry, this matrix admits the isotropic-plus-rank-one decomposition
\[
H(\beta^\star)
=
\E[f'(Z)^2]\mathbf{I}_d
+
\left(
\E[Z^2f'(Z)^2]
-
\E[f'(Z)^2]
\right)
\beta^\star{\beta^\star}^{\top},
\]
where $Z\sim\mathcal{N}(0,1)$. Its eigenvalue in the direction $\beta^\star$ is $\E[Z^2f'(Z)^2]$, while every direction orthogonal to $\beta^\star$ has eigenvalue
$\E[f'(Z)^2]$. Therefore, $\lambda_{\min}(H(\beta^\star))
=\mu$, and $\lambda_{\max}(H(\beta^\star))=\mu_1,$
where $\mu$ and $\mu_1$ are respectively the minimum and maximum of these two one-dimensional Gaussian expectations.

\textbf{Perturbation Analysis.}
Let $z:=X^\top\beta$, and $z^\star:=X^\top\beta^\star$. The Hessian perturbation can be written as $\Delta(\beta)
=
\E\!\left[
\Delta Q(z,z^\star)XX^\top
\right]$, 
where
\[
\Delta Q(z,z^\star)
=
f'(z)^2-f'(z^\star)^2
+
f''(z)\bigl(f(z)-f(z^\star)\bigr).
\]
To expose the dependence on $\beta-\beta^\star$, consider the line segment $\beta_t:=(1-t)\beta^\star+t\beta$, $
z_t:=X^\top\beta_t$, $t\in[0,1]$. Applying the Fundamental Theorem of Calculus separately to the two differences above gives the exact representation $\Delta Q(z,z^\star)
=
X^\top(\beta-\beta^\star)
\int_0^1 C_t\,dt$,
where
\[
C_t
=
2f'(z_t)f''(z_t)
+
f''(z)f'(z_t).
\]

For a unit vector $v$, the quadratic form $v^\top\Delta(\beta)v$ is controlled by three ingredients.
First, Gaussianity bounds the mixed moment involving $X^\top(\beta-\beta^\star)$ and $(X^\top v)^2$. Second, H\"older's inequality and the definitions of $M_1(R)$ and $M_2(R)$ give $\norm{C_t}_{L^2}\leq 3M_1(R)M_2(R)$ uniformly over $t\in[0,1]$, because the entire segment
$\{\beta_t\}_{t\in[0,1]}$ remains inside $\mathcal{B}(\beta^\star,R)$. Third, Minkowski's integral inequality transfers this uniform estimate to
$\int_0^1 C_t\,dt$. Taking the supremum over all unit vectors yields
\[
\norm{\Delta(\beta)}_{\mathrm{op}}
\leq
3\sqrt{15}\,
M_1(R)M_2(R)
\norm{\beta-\beta^\star}_2.
\]

\textbf{Establishing the Convex Basin.}
Under \cref{assump:con-basin}, $6\sqrt{15}\,R\,M_1(R)M_2(R)\leq\mu,$
and hence $\norm{\Delta(\beta)}_{\mathrm{op}}\leq\frac{\mu}{2}.$ Weyl's inequality then gives $\lambda_{\min}(H(\beta))
\geq
\frac{\mu}{2}$, and $\lambda_{\max}(H(\beta))\leq \mu_1+\frac{\mu}{2}$.
Thus,
$\frac{\mu}{2}\mathbf{I}_d \preceq H(\beta)\preceq\left(\mu_1+\frac{\mu}{2}\right)\mathbf{I}_d$ throughout $\mathcal{B}(\beta^\star,R)$.
Finally, $M_1(R)$ and $M_2(R)$ depend only on one-dimensional Gaussian moments, since
$X^\top u\sim\mathcal{N}(0,\norm{u}_2^2)$ for every fixed $u$. Hence, the certified radius is independent of the dimension $d$.
\end{proof}

\section{Linear-Time Robust Recovery of SIMs}\label{sec:lin}
In this section, we present our linear-sample and time algorithm for robustly recovering the true signal $\beta^\star$ when the link function admits a convex basin in the loss landscape under heavy-tailed noise and strong adversarial contamination. 
We now formally state our main result.
\begin{theorem}[Linear-time Algorithm for Robust Recovery]\label{thm:lintime}
Consider the model in \cref{def:model}. 
Define $C_{\mathrm{lip}}(R)$ and $R$ as in \cref{lem:local_geometry}. Define $\alpha=\frac{\mu}{2}+\mu_{1}, \gamma=\frac{\mu}{2}$ denote the smoothness and strong convexity parameters of \cref{lem:local_geometry}. Define $$\phi_{1}:=\sup_{\beta\in \mathcal{B}(\beta^\star, R)}\mathbb{E}\!\left[f'(X^{\top}\beta)^{16})\right]^{1/4},$$ and $$\phi_{2}:=\sup_{\beta\in \mathcal{B}(\beta^\star, R)}\mathbb{E}\!\left[f'(X^{\top}\beta)^{4}\right]^{1/2},$$ and assume $K_{4}\leq K$. Define $c := \mathrm{ESC}(\beta^\star; f)$, and let $C_{4}$ be the hypercontractivity parameter as defined in \cref{lem:hypercontracitivity of YX}.
\cref{alg:lineartimealgo} takes $n=\tilde{O}(m+P\tilde{m})$ samples from an 
    $\epsilon$-contaminated dataset $T$, such that $$\epsilon=O\left(\min\left\{\frac{1}{C_{4}^{4}}, \frac{ c^2 \min\{R^4,1\}}{C_{4}^4\left(
\sigma^2 + \mathbb{E}[f^2] + c
\right)^2},\frac{\gamma^2}{\phi_{1}},\frac{\gamma^2 R^2}{\sigma^2\phi_{2}}\right\}\right),$$ and outputs an estimate $\beta$ of the true parameter in time $\tilde{o}\left(\frac{md}{C_{4}^{4}}+P\tilde{m}d\right)$ w.h.p., s.t.
    \begin{equation}
        \norm{\beta-\beta^\star}= O(\sigma\sqrt{\epsilon}),
    \end{equation} 
where $m=\Theta\left(C_{4}^2
\cdot\left(\frac{d \log d + \log(1/\delta)}{\epsilon^{3/2}}\right)
\right), \tilde{m}=\tilde{O}(d/\epsilon)$ and $P=O(1)$ denotes the number of iterations of the \textsc{LRGD} algorithm (see \cref{algorithm:robust gradeint descent_linear}).

\end{theorem}
In the above \cref{thm:lintime}, $m$ is the sample complexity for the robust spectral initialization subroutine (see \cref{alg:robust_spectral_init_lin_time}) and $\tilde{m}$ is the sample complexity for the robust gradient descent subroutine (see \cref{algorithm:robust gradeint descent_linear}).
We begin by giving the pseudocode of \cref{alg:lineartimealgo} in \cref{thm:lintime} along with a high-level overview.
\begin{algorithm}[h]
\label{algo:lintime}
\caption{Linear-time Algorithm for Robust Recovery}
\label{alg:lineartimealgo}
\begin{algorithmic}[1]
\STATE \textbf{Input:} Samples $S = \{(x_i, y_i)\}_{i=1}^N$, Corruption $\epsilon$, parameters $P$, $\alpha, \gamma$.
% \STATE \textbf{Output:} Estimator $\beta_P/\norm{\beta_{P}}.$
\STATE Randomly partition the $N$ samples into $P+1$ disjoint buckets of equal sizes, denoted by $
N_1, N_2, \ldots, N_{P+1}.
$
\STATE $\beta_0 \gets \textsc{LRSI}(N_{1},\epsilon)$ \COMMENT{Initialize in convex basin}
\STATE $\beta_{P} \gets \textsc{LRGD}( N_{2}\ldots N_{P+1},\beta_{0},\epsilon, \alpha, \gamma)$ 
\STATE \textbf{Output:} $\beta_P/\norm{\beta_{P}}.$
\end{algorithmic}
\end{algorithm}

First, we begin by recalling that if our link function $f$ satisfies \cref{assump:con-basin,assump:sq-curv}, the population loss-landscape admits a convex basin in a neighborhood of $\beta^\star$ (see \cref{lem:local_geometry}). \cref{alg:lineartimealgo} exploits this structure via the \textsc{LRSI} subroutine (see~\cref{alg:robust_spectral_init_lin_time}) which combines generalized higher-order Stein’s  identities (see \cref{lem:steins}) together with the linear-time robust hypercontractive $1$-ePCA algorithm as described in \cref{thm:robust_hypercontractive_kePCA} to construct an estimate $\beta_0$ of the true signal $\beta^\star$ that lies within the convex basin. Finally, using $\beta_0$ as a warm start, the \textsc{LRGD} algorithm (see~\cref{algorithm:robust gradeint descent_linear}) converges to $\beta^\star$.
We first carefully detail the robust spectral initialization step, then the robust gradient descent step, and finally the proof of \cref{thm:lintime}.
\subsection{Warm Start via Linear time Spectral Initialization}
\begin{restatable}{lemma}{lemSpectralIni}
\label{lem:spectral}
    Consider the model given in \cref{def:model} and define $\Tilde{Y}:=YX$. Then, $\beta^\star$ is the top eigenvector of $\expect{\Tilde{Y}\Tilde{Y}^T}$  with eigenvalue $\lambda_{\max}= \sigma^2+\expect{(f({X^T\beta^\star}))^2}+2\expect{(f^{\prime}{X^T\beta^\star})^2+f({X^T\beta^\star})\cdot f^{\prime\prime}({X^T\beta^\star})}.$
\end{restatable}

\Cref{lem:spectral} and \cref{assump:sq-curv} state that for the random variable $\widetilde{Y} = YX$, where $Y,X$ are defined in \cref{def:model} s.t. the link function $f$ satisfies \cref{assump:sq-curv}, the vector $\beta^\star$ is the leading eigenvector of the matrix $\mathbb{E}[\widetilde{Y}\widetilde{Y}^\top]$, with eigenvalue 
$\lambda_{\max}$. For a warm start, we robustly estimate the leading eigenvector of the second-moment matrix of $\widetilde{Y}$. To this end, we employ the robust hypercontractive $1$-ePCA algorithm. The applicability of this method requires the distribution of $\widetilde{Y}$ to be hypercontractive, a property we establish in \cref{lem:hypercontracitivity of YX}.
\begin{restatable}{lemma}{hypercontrativityofYX}
 \label{lem:hypercontracitivity of YX}
 Consider the model in \cref{def:model}. Then, $\widetilde{Y} = YX$ is $(4,C_{4})$ hypercontractive, where $$C_{4}={3\left(\mathbb{E}\left[ f\!\left(X^\top \beta^\star\right)^8\right]^{1/8}+ K_{4}\right)}/{\sigma}.$$
\end{restatable}

We now state the guarantees for the \textsc{LRSI} algorithm (see \cref{alg:robust_spectral_init_lin_time}) that uses the PCA algorithm of \cite{jambulapati2024black} as a subroutine.
\begin{algorithm}[h]
\caption{Linear-Robust-Spectral-Initialization \textsc{(LRSI)}}
\label{alg:robust_spectral_init_lin_time}
\begin{algorithmic}[1]
\STATE \textbf{Input:} Sample sets $N_1$, corruption level $\epsilon$.
\STATE \textbf{Output:} Initial estimate $\beta_0$
\STATE Consider the samples $N_{1}$, define $X'_j=y_jx_j$. Apply the Robust hypercontractive 1-ePCA algorithm to $\{X'_j\}$ and let $\hat{u}$ be the top eigenvector estimate.
\STATE \textbf{Return} $\beta_0\gets\hat{u}$.
\end{algorithmic}
\end{algorithm}
\begin{restatable}[Linear-time algorithm for spectral initialization]{theorem}{lintimespecinitalization}
   \label{thm:lin-time-spec-ini}
Consider \cref{def:model}. Let $c=\mathrm{ESC}(\beta^\star; f)$, $\delta\in(0,1)$. Let $C_{4}$ be hypercontractivity constant of $\tilde{Y}=YX$ as defined in \cref{lem:hypercontracitivity of YX}. For contamination parameter $\epsilon=O\left(\min\{\frac{1}{C_{4}^{4}},\frac{c^2}{C_{4}^4
\left(
\sigma^2 + \mathbb{E}[f^2] + c
\right)^2}\}\right)$, w.p. $\geq 1-\delta$, the \cref{alg:robust_spectral_init_lin_time} takes time $O\left(
\frac{m d }{C_{4}^4}
\operatorname{polylog}\left(\frac{d}{\epsilon \delta}\right)
\right)$ and $m=\Theta\left(C_{4}^2
\frac{d \log d + \log(1/\delta)}{\epsilon^{3/2}}
\right)$ samples to output a unit norm  vector $\beta_0$ s.t. \[\operatorname{dist}(\beta_0, \beta^\star) =O\left(
\frac{
C_4 \epsilon^{\frac{1}{4}}\sqrt{
\sigma^2 + \mathbb{E}\!\left[f\!\left(X^\top \beta^\star\right)^2\right] + c}
}{
\sqrt{c}
}
\right).\] 
\end{restatable}

\begin{proof}[Proof Sketch]
\cref{lem:hypercontracitivity of YX} establishes that $\widetilde{Y}=YX$ follows a $(4,C_4)$-hypercontractive distribution, where the constant $C_4$ is specified in \cref{lem:hypercontracitivity of YX}. This property allows us to directly invoke \cref{thm:robust_hypercontractive_kePCA},
which yields the guarantees for the spectral initialization step of our algorithm.
\end{proof}
\begin{remark}
    Note that the ESC assumption (\cref{assump:sq-curv}) alone guarantees the existence of an estimator $\beta_{0}$ such that $\|\beta_{0} - \beta^{\ast}\| = O\!\left(\epsilon^{1/4}\right).$
\end{remark}
\subsection{Linear Robust Gradient Descent}
The second key step, following phase retrieval, is to perform gradient descent on the population risk, $\beta_{t+1} = \beta_t - \eta \nabla \mathcal{L}(\beta_t).$
By Theorem~3.12 of \citet{bubeck2015convex} together with Lemma~\ref{lem:local_geometry}, the iteration stated above converges linearly to the global minimizer, provided that all iterates remain within the ball $\mathcal{B}(\beta^\star,R)$ and the step size is chosen as $\eta = \frac{2}{\alpha+\gamma}$, where $\alpha$ and $\gamma$ denotes the smoothness and the strong convexity parameters, respectively. \cref{lem:local_geometry} implies that, for the population loss $\mathcal{L}(\beta)$, the strong convexity parameter is $\gamma =\frac{\mu}{2}$ and the smoothness parameter is $\alpha = \frac{\mu}{2} + \mu_1$.
Since the learning algorithm only has access to the dataset and not the population gradient $\nabla \mathcal{L}(\beta)$, the update stated above cannot be implemented directly. To address this, we follow the standard approach of expressing the gradient as an expectation \cite{prasad2020robust,bunaR24_robust_phase}. In particular, $\nabla \mathcal{L}(\beta)
=
\mathbb{E}\!\left[(f(X^\top \beta)-Y)f'(X^\top \beta)X\right].$
%\begin{equation}\label{eq:grad-expectation}
%\nabla \mathcal{L}(\beta)
%=
%\mathbb{E}\!\left[(f(X^\top \beta)-Y)f'(X^\top \beta)X\right].
%\end{equation}
Substituting the value of $\nabla \mathcal{L}(\beta)
$ into above update yields
$$\beta_{t+1}
=
\beta_t
-
\eta\,\mathbb{E}\!\left[(f(X^\top \beta_t)-Y)f'(X^\top \beta_t)X\right].$$
We then replace the expectation with a robust estimator of the gradient. We state the definition of a robust gradient estimator (see Definition 2.1.1 in \cite{bunaR24_robust_phase}) below. 
\begin{definition}\label{definition:rodust estimator of gradient}[Robust Gradient Estimator]
  Consider a sample \( T = \left\{ \left( x_i, y_i \right) \right\}_{i=1}^m \) of size \( m \). We call \( g(\cdot; T, \delta, \epsilon) \) a gradient estimator if there exist functions \( A \) and \( B \), where \( A, B: \mathbb{N} \times [0,1]^2 \to \mathbb{R} \), such that for any fixed point \( \beta \in \mathbb{R}^n \), w.p. \(\geq 1 - \delta \), 
  \[
  \|g(\beta ; T, \delta, \epsilon)-\nabla r(\beta)\| \leq A(m, \delta, \epsilon)\left\|\beta-\beta^\star\right\|+B(m, \delta, \epsilon).
  \]
\end{definition}
Following \citet{bunaR24_robust_phase}, we use the notion of a robust gradient estimator, stated formally in Definition~\ref{definition:rodust estimator of gradient}. Let $
g_t := g(\beta_t; T, \delta, \epsilon)
$
denote such an estimator computed from the dataset \(T\). The resulting robust gradient descent update is $\beta_{t+1} = \beta_t - \eta g_t $.
%\begin{equation}\label{eq:robust-gd}
%\beta_{t+1} = \beta_t - \eta g_t .
%\end{equation}
We summarize the resulting robust gradient descent procedure in Algorithm~\ref{algorithm:robust gradeint descent_linear}, and then state its main guarantees.
\begin{algorithm}[H]
\caption{Linear-Robust-Gradient-Descent \textsc{(LRGD)}}
\label{algorithm:robust gradeint descent_linear}
\textbf{Inputs:} $\beta_{0}, \delta \in (0, 1), \epsilon > 0, P \in \mathbb{N}, \alpha, \gamma$ and datasets $N_2, \ldots, N_{P+1}$.\\
\textbf{Output:} $\beta_P/\|\beta_P\| \in \mathbb{R}^n$
\begin{algorithmic}[1]
\STATE Set $\eta =\frac{2}{\alpha+\gamma}=\frac{2}{\mu+\mu_{1}}$.
For {$t = 0, \ldots, P-1$}:
    \newline$\triangleright$ Receive contaminated samples $B_t = \left\{\left(x_{j}, y_j\right)\right\}_{j=1}^{\tilde{m}}$. 
    \newline$\triangleright$ \textbf{Gradient Estimation:} For each $\left(x_j, y_j\right) \in B_t$,
 compute $p_{j}^{t} = \left(f\left(x_j^{\top }\beta_t\right) - y_j\right)f'\left(x_j^{\top} \beta_t\right) x_j$. 
    \newline$\triangleright$ Compute $g_t$, the robust mean estimate for $\left\{p_{t}^{j}\right\}$ using Robust Mean Estimation (\cref{thm:lin-tim-robut-Mean Value Theorem}).
    \newline$\triangleright$ Update $\beta_{t+1} = \beta_t - \eta g_t$.
\STATE \textbf{Return} $\frac{\beta_P}{\|\beta_P\|}$.
\end{algorithmic}
\end{algorithm}
\begin{restatable}{theorem}{linRobustGradientDescnt}
    \label{theorem: li-time-grads-descent}Consider \( R, \mu\) and $\mu_{1}$ as defined in \cref{lem:local_geometry}. 
    Define $\alpha$, $\gamma$, $\phi_{1}$, and $\phi_{2}$ as in \cref{thm:lintime}.
    Let $\beta_{0} \in \mathcal{B}(\pm \beta^\star,R)$ and contamination parameter $$\epsilon=O\left( \min\left\{\frac{\gamma^2}{\phi_{1}},\frac{\gamma^2 R^2}{\sigma^2\phi_{2}}\right\}\right).$$ \cref{algorithm:robust gradeint descent_linear} takes time ${O}\!\left(P\tilde{m}d\log^4\left(\frac{d}{\epsilon \delta}\right)\right)$ and samples $O(P\tilde{m})$ to output an unit norm  vector $\beta^{(P)}=\frac{\beta_{P}}{\|\beta_{P}\|}$, with probability at least $1-P\delta$, s.t., $$\left\|\beta^{(P)}-\beta^\star\right\|\leq  2R \exp \left(- P\left(\frac{ \gamma}{\alpha+\gamma}\right)\right)+ O\!\left(\frac{\sigma \sqrt{\phi_{2}\cdot\epsilon}}{\gamma}\right)$$
% \begin{align*}
% &,
% \end{align*}
where $
\begin{aligned}[t]
    \tilde{m}=\tilde{O}\left(d/\epsilon\right),
\end{aligned}$
and $P=O(1)$ is the number of time-steps in \cref{alg:robust_spectral_init_lin_time}.
\end{restatable}
\begin{proof}[Proof Sketch]
We follow a standard inductive argument for gradient descent~\cite{prasad2020robust,bunaR24_robust_phase}. The objective is to show that the distance (error) to the true signal decreases at each iteration. To establish this, we first derive bounds on the trace and operator norm of the covariance matrix of the gradient of the loss function (\cref{lemma: calculations of trace and operator norm of gradients}). We then relate the distance to the true signal at the $(t+1)$-th iterate to that at the $t$-th iterate (\cref{lemma:relation of current iterate to previous one}). In particular, we show that each iterate satisfies the definition of a robust gradient (\cref{definition:rodust estimator of gradient}). Finally, we combine these bounds to control the total error across all iterations, as shown in \cref{eq19} of the paper.
% The key step is bounding the operator norm of the covariance matrix of the gradient (see \cref{lemma: calculations of trace and operator norm of gradients}). The remaining arguments then follow from the standard analyses of
% \cite{prasad2020robust,bunaR24_robust_phase}, together with \cref{thm:lin-tim-robut-Mean Value Theorem}.
\end{proof}
\textbf{Proof of \cref{thm:lintime}.}
    Under the assumptions on $m$ and $\epsilon$, the output $\beta_{0}$ of the LRSI algorithm
(see \cref{alg:robust_spectral_init_lin_time}) satisfies
$
\|\beta_{0} - \beta^{\ast}\|
=
O\!\left(
\frac{C_{4}\bigl(\sigma^{2} + \mathbb{E}[f^{2}] + c\bigr)^{1/2}\epsilon^{1/4}}{\sqrt{c}}
\, 
\right).
$ Moreover, under the additional assumption on the corruption level
$
\epsilon \leq
\nicefrac{R^{4} c^{2}}{C_{4}^{4}\bigl(\sigma^{2} + \mathbb{E}[f^{2}] + c\bigr)^{2}},
$
we have $\|\beta_{0} - \beta^{\ast}\| \leq R$. The proof now follows directly from
\cref{theorem: li-time-grads-descent}.

\subsection{Applications}\label{subsec:corr-lin}
In this section, we now demonstrate near-linear time and sample robust recovery for SIMs with $6$ different link functions under heavy-tailed noise and strong adversarial contamination. Our representative link functions can be broadly classified into three categories: (i) \textbf{Monotonic Links:}
    Logistic/Sigmoid ($\sigma(z)$), Tanh ($\tanh(z)$), and Probit ($\Phi(z)=\int_{-\infty}^z \phi(u) du$). Here, $\phi(z) = \frac{e^{-z^2/2}}{\sqrt{2\pi}} $ and $\sigma(z) = \frac{1}{1+e^{-z}}$, (ii) Phase Retrieval ($f(z) = z^2$), and (iii) \textbf{Asymmetric Non-monotonic Links:} GeLU ($z\Phi(z)$), and 
    Swish ($z\sigma(z)$).
% \begin{table}[h!]
%     \centering
%     \scriptsize
%     \caption{The explicit values of parameters needed for \cref{thm:lintime} for different link functions.}
%     \begin{tabular}{lccccccc}
%         \toprule
%         \textbf{Function} & \textbf{ESC} & $\boldsymbol{\mu}$ & $\boldsymbol{\mu_1}$ & $\boldsymbol{R}$ & $\boldsymbol{C_{lip}(R)}$ & $\boldsymbol{\phi_1}$ & $\boldsymbol{\phi_2}$ \\
%         \midrule
%         Logistic/Sigmoid & 0.03120 & 0.02371 & 0.04484 & 0.37070 & 0.00759   & 0.00327   & 0.05418 \\
%         Tanh             & 0.18180 & 0.10781 & 0.46440 & 0.00477 & 2.68103   & 0.64787   & 0.58561 \\
%         Probit           & 0.04594 & 0.03063 & 0.09189 & 0.04731 & 0.07683   & 0.01798   & 0.10850 \\
%         Phase Retrieval  & 6.00000 & 4.00000 & 12.00000 & 0.00164 & 288.94351 & 607.68102 & 6.95090 \\
%         GeLU             & 0.48648 & 0.45585 & 0.57837 & 0.02940 & 1.83990   & 1.00965   & 0.66379 \\
%         Swish            & 0.41718 & 0.37948 & 0.51748 & 0.05197 & 0.86660   & 0.77467   & 0.54743 \\
%         GeGLU            & 2.93420 & 1.97655 & 6.19511 & 0.00144 & 163.19197 & 513.19663 & 5.07447 \\
%         SwiGLU           & 2.58764 & 1.70761 & 5.69968 & 0.00137 & 148.37378 & 543.98243 & 4.67740 \\
%         \bottomrule
%     \end{tabular}
%     \label{table:explicitvalues}
% \end{table}
\begin{corollary}\label{corr:lintime}
    Consider the model in \cref{def:model}, with the following link functions: Phase Retrieval, GeLU, Swish, Tanh, Probit, and Logistic. There exists an algorithm that, using $n=\tilde{O}(d)$ samples and tolerating a constant fraction $\epsilon$ of adversarial contamination under heavy-tailed noise with variance $\sigma^2$, outputs an estimate $\hat{\beta}$ s.t.
 $\hat{\beta}$ satisfying $\|\hat{\beta} - \beta^\star\|_2 = O(\sigma\sqrt{\epsilon})$ in $\tilde{O}(n d)$ time, with high probability.
\end{corollary}
\begin{proof}
    The proof follows from the fact that of the above link functions satisfy \cref{assump:con-basin,assump:sq-curv} (see Table 2 in \cref{sec:numericals}) and \cref{thm:lintime}.
\end{proof}

\section{Discussion and Future Work}\label{sec:dis}
In this work, we established the first framework achieving \textbf{near-linear time and optimal sample complexity} for the robust recovery of Single-Index Models with generic, non-monotonic link functions (e.g., GeLU, Swish) under heavy-tailed noise and strong adversarial contamination. 
Below we outline a few interesting future directions.

\textbf{Optimal Error Rates.} Our estimator achieves an $\ell_2$ error rate of $O(\sigma\sqrt{\epsilon})$ under Gaussian covariates in the presence of heavy-tailed noise and $\epsilon$-fraction adversarial contamination. In comparison, for the same setting, \citet{Pensia,cherapanamjeri2020optimalrobustlinearregression} established the information-theoretically optimal rate $\tilde{O}(\sigma\epsilon)$. However, for non-linear models with non-convex population loss, existing provably robust algorithms (e.g., phase retrieval), are only known to achieve a $O(\sigma\sqrt{\epsilon})$ rate~\cite{bunaR24_robust_phase,das2026}. Our result matches this best-known rate for non-linear single-index models while accommodating a significantly broader class of link functions, including non-monotonic ones. Closing the gap between the achievable rate and the information-theoretically optimal rate for general SIMs remains an important open problem and we leave it for future work.

\textbf{Non-Gaussian Covariates.} 
Our theoretical guarantees heavily leverage the Gaussianity of the design matrix to derive the ESC condition, and in our basin radius analysis by exploiting rotational invariance. We leave extending our proofs to even sub-Gaussian designs as an open question.

\textbf{Alternative Loss Landscapes and Adversary Models.} While we focused on the squared loss, investigating if convex basins persist under Huber loss or general $M$-estimators remains an important open question. Additionally, since the existence of a convex-basin does not depend on the corruption model, adapting our framework to Agnostic Learning or Differential Privacy settings is a promising future direction.

\textbf{Multi-Index Models (MIMs)} 
As noted in our introduction, functions like GeLU and Swish are scalar primitives for GLUs, which are inherently Multi-Index Models (MIMs) defined by interactions between multiple projections, as $y = \langle \beta_1, x \rangle \cdot f(\langle \beta_2, x \rangle)$. Extending the  guarantees of our work to MIMs (particularly \cref{assump:con-basin,assump:sq-curv}) requires disentangling the interaction terms between multiple weight vectors. This likely necessitates robust tensor decomposition techniques of order significantly higher than those required for SIMs, which presents a distinct set of algebraic and algorithmic challenges.

\textbf{Robust Recovery for Links with Information Exponent $(k \geq 3)$.}
Our framework primarily targets link functions where the signal is detectable via low-order derivatives (specifically, where \textsc{ESC} is non-trivial). However, for link functions with an information  exponent $k \geq 3$, the signal is entirely suppressed in lower-order moments. Robustly recovering $\beta^\star$ in this regime would require working with higher-order moment tensors (of order at least $k$). This is an interesting avenue for future work and one concrete line of investigation is discussed next.

\textbf{Identifying Label Transforms.}
Suppose there exists a map $\tau:\mathbb{R}\to\mathbb{R}$ such that $\tilde{f}=\tau\circ f$ has IE $k^\star\leq 2$ and satisfies \cref{assump:con-basin,assump:sq-curv}. Then our 
framework applies directly to $\tilde{f}$, yielding efficient robust recovery for the original link function $f$. Identifying such transforms is highly non-trivial, since $k^\star$ is the infimum of the IE over all square-integrable label transforms 
\citep[Proposition 2.6]{damian2024computational}, and characterizing which link functions admit a transform $\tau$ such that the resulting $\tilde{f}$ has IE $\leq 2$ 
and satisfies \cref{assump:con-basin,assump:sq-curv} is an open problem. We view this as a promising direction for future work, and note that our results provide the first motivation for investigating such regularity conditions on label transforms.

\textbf{Empirical Verification.} 
While our focus is theoretical, empirical evaluation is a vital next step. Implementing high-order robust spectral estimators involves practical engineering challenges, particularly regarding numerical stability and hyperparameter tuning for the filtering subroutines. We leave extensive empirical benchmarking on real-world datasets and the development of practical heuristics based on our theory for future work.

\section*{Acknowledgements}
 The authors are grateful to Ankit Pensia, for many valuable discussions and in particular for pointing us to the robust PCA and robust mean estimation subroutines used in \cref{sec:lin} which was instrumental in obtaining the near-linear time guarantee of \cref{thm:lintime}. The authors also thank Tanmay Goyal for pointing out an error in an earlier version of the proof of Theorem 3.1. The authors also thank the anonymous reviewers of ICML 2026 for their thorough and constructive feedback, which helped improve the presentation of this work. This work was supported by the Department of Atomic Energy, Government  of India, under project no. RTI4014.

% \section*{Impact Statement}
% This paper presents work whose goal is to advance the field of Machine Learning. There are many potential societal consequences of our work, none which we feel must be specifically highlighted here.

\bibliography{refs}
\bibliographystyle{icml2026}

\newpage
\appendix
\onecolumn

\section{Literature Review}\label{sec:litreview}
\begin{table}[ht]
\centering
\caption{Comparative Analysis of Single Index Model Architectures. \textbf{Adv. Rob.} stands for Adversarial Robustness. \textbf{H.T.} stands for Heavy-Tailed Noise. For both of these columns, \cmark~denotes Yes, \xmark ~denotes No, and \pmark~denotes only label corruptions. \textbf{Rank $k$} refers to the first non-zero Hermite coefficient ($k=1$: Linear, $k=2$: Quadratic). For \textbf{Sample/Time Complexity} columns, $n$ is the sample size, $d$ is the dimension, and $s$ is sparsity.
}
\label{tab:comparison}

\resizebox{\textwidth}{!}{%
\begin{small}
\begin{tabular}{@{}l l c c l l @{}}
\toprule
\textbf{Reference} & \textbf{Link Function} & \textbf{Adversarial} & \textbf{H.T.}  & \textbf{Sample} & \textbf{Time} \\
& & \textbf{Robustness} & \textbf{Noise} &  \textbf{Complexity} & \textbf{Complexity} \\
\midrule

% --- Baseline 1: Monotonic / Benign ---
\citet{kalai2009isotron} [Isotron] & Monotone SIM & \xmark & \xmark  & Poly$(d)$ & Poly$(d)$ \\
\citet{plan2017high} & Odd SIM & \xmark & \xmark  & $\tilde{O}(d)$ & $O(d)$. \\

% --- Baseline 2: Phase Retrieval / Benign ---
\citet{NetrapalliJS15} & Phase Retrieval & \xmark & \xmark  & $\tilde{O}(d)$ & $\tilde{O}(nd)$ \\
\citet{candes2015phase} & Phase Retrieval & \xmark & \xmark  & $\tilde{O}(d)$ & $\tilde{O}(nd)$ \\
\citet{zhang2018median} & Phase Retrieval & \pmark & \xmark  & $\tilde{O}(d)$ & $\tilde{O}(nd)$ \\
\citet{lu2020phase} & SIM & \xmark & \xmark  & $\tilde{O}(d)$ & $O(nd)$ \\
\citet{jagatap2017fast} & Phase Retrieval & \xmark & \xmark  & $\tilde{O}(s^2)$ & $\tilde{O}(ns^2 )$ \\

% --- Robust Methods ---
\citet{klivans18a} & Poly. Regression & \cmark & \cmark  & Poly$(d)$ & Poly$(d)$ (SoS) \\
\citet{Sever} & Monotone SIM & \cmark & \cmark  & $\tilde{O}(d)$ & Poly$(d)$ \\
\citet{dong2025outlier} & Phase Retrieval & \cmark & \xmark  & $\tilde{O}(d)$ & $\tilde{O}(nd)$ \\
\citet{pensiastream} & Logistic Regression & \cmark & \cmark  & $\tilde{O}(d^2)$ & $\tilde{O}(nd)$ \\
\citet{awasthi2022trimmed} & Monotone SIM & \cmark & \xmark  & $\tilde{O}(d)$ & Heuristic \\
\citet{huang2025robust} & Phase Retrieval & \pmark & \xmark  & $O(d)$ & Poly$(d)$ (SDP) \\
\citet{diakonikolas2019efficient} & Linear Regression & \cmark & \xmark  & $\tilde{O}(d)$ & Poly$(d)$ \\

% --- Recent SIM / Non-Convex Analysis ---
\citet{damian2024smoothing} & SIM & \xmark & \xmark  & $O(d+d^{k^{*}/2})$ & $\tilde{O}(nd)$ \\
\citet{bunaR24_robust_phase} & Phase Retrieval & \cmark & \cmark  & $\tilde{O}(d)$ & $e^{d}$  \\
\citet{das2026} & Phase Retrieval & \cmark & \cmark  & $\tilde{O}(d)$ & $\tilde{O}(n^2d)$ \\

\midrule
\rowcolor{blue!10} \textbf{Our Proposed Framework} & \textbf{SIM with positive ESC} & \cmark & \cmark & $\mathbf{\tilde{O}(d)}$ & $\mathbf{\tilde{O}(n d)}$ \\
\bottomrule
\end{tabular}
\end{small}
}
\end{table}
\textbf{Monotone link functions without corruption and heavy-tailed noise:}
 \cite{kalai2009isotron} study the recovery of monotone Lipschitz single-index models (SIMs) with an unknown link function $f$ via the Isotron algorithm, which operates in $\mathrm{poly}(d)$ time with $O(1)$ sample complexity. For arbitrary link functions, \citet{plan2017high} provide a linear-time estimator requiring \(O(d)\) samples for consistent recovery. Their approach, however, relies on the non-degeneracy condition $\mathbb{E}\!\left[f(X^{\top}\beta^*+\zeta)\,X^{\top}\beta^*\right]\neq 0,$
and therefore does not apply when \(f\) is an even function and the noise distribution is symmetric. Further extending this scope, \cite{damian2024smoothing} analyze link functions with polynomial-tailed derivatives using online stochastic gradient descent on a smoothed loss function. Their approach achieves a sample complexity of $n = O(d + d^{k^*/2})$ and time complexity of $O(nd)$, where $k^*$ denotes the \textit{information exponent} (the smallest nonzero coefficient of the link function's Hermite expansion under the Gaussian measure).

\textbf{Monotone link function under adversarial corruption and heavy-tailed:}
\citet{klivans18a} study linear and polynomial regression with Gaussian covariates, and propose a Sum-of-Squares–based algorithm achieving polynomial sample and computational complexity in $d$. \citet{Sever} study robust learning in two settings: for SIMs with a monotone (logistic) link, they give a polynomial-time algorithm with sample complexity $\tilde{O}(d)$ achieving generalization error $O(\epsilon^{1/4})$ under logistic loss; for linear regression with heavy-tailed noise, they propose a polynomial-time algorithm with sample complexity $O(d^5)$ achieving estimation error $O(\sqrt{\epsilon})$, where $\epsilon$ denotes the contamination level. \cite{diakonikolas2019efficient} study linear regression under Gaussian noise, and propose a polynomial-time algorithm achieving sample complexity $\tilde{O}(d)$ and estimation error optimal up to constant factors. 
\citet{awasthi2022trimmed} study monotone SIMs under adversarial contamination and propose a algorithm based on trimmed MLE, achieving sample complexity $\tilde{O}(d)$ with heuristic time complexity.

\textbf{Phase retrieval:} Phase retrieval is a classical inverse problem with applications in optics,
crystallography, and X-ray imaging. In the non-contaminated, light-tailed setting,
the first provable algorithm was due to \cite{NetrapalliJS15}, followed by a line of
work including \cite{codeddiffraction,jagatap2017fast} and others. \citet{NetrapalliJS15} propose an alternating minimization approach, while
\cite{candes2015phase} use a gradient-based method; both achieve sample complexity
$\widetilde{O}(d)$ and runtime $\widetilde{O}(nd)$. Notably, the method of
\citet{NetrapalliJS15} requires fresh samples at each iteration, whereas
\cite{candes2015phase} does not. Whereas, \citet{jagatap2017fast} study sparse phase retrieval ($\beta^\star$ is $s$ sparse) and propose an algorithm, achieves sample complexity $\tilde{O}(s^2)$ and computational complexity $\tilde{O}(s^2n)$. \citet{zhang2018median} study phase retrieval under dense bounded noise and label
corruption using a median-truncated Wirtinger Flow algorithm with sample complexity
$\widetilde{O}(d)$ and runtime $\widetilde{O}(nd)$, while \cite{huang2025robust}
consider same settings as \cite{zhang2018median} and
propose an SDP-based algorithm achieving sample complexity $O(d)$. Phase retrieval under strong adversarial contamination of both labels and covariates
has been studied in \cite{dong2025outlier,bunaR24_robust_phase,das2026}.
Specifically, \cite{dong2025outlier} consider the noiseless setting and propose a
linear-time algorithm with sample complexity $\widetilde{O}(d)$, while
\cite{bunaR24_robust_phase} study heavy-tailed label noise and give an
exponential-time algorithm achieving sample complexity $\widetilde{O}(d)$.
Subsequently, \cite{das2026} obtain the same optimal sample complexity $n=\tilde{O}(d)$ with runtime
$\widetilde{O}(n^{2}d)$. 

\section{Details and Omitted Proofs of \cref{sec:con-basin}}\label{sec:proofscon-basin}
\subsection{Gaussian symmetry}
In this section, we state and prove two important properties of the Gaussian
distribution, which will be used later in the proof of \cref{lem:local_geometry}.
\begin{lemma}\label{lem:Gaussian_symmetry}
Let $X \sim \mathcal N(0,\mathbf{I}_d)$, $g:\mathbb{R}\rightarrow \mathbb{R}$, $Z\sim\mathcal{N}(0,1)$ and assume $\|\beta^\star\|_2 = 1$. Then,
\[\mathbb{E}[g(X^{\top}\beta^\star)XX^{\top}]=\mathbb E[g(Z)] \mathbf{I}_d
+
\Big(
\mathbb E[Z^2 g(Z)]
-
\mathbb E[g(Z)]
\Big)\beta^\star\beta^{*\top}.\]
\end{lemma}
\begin{proof}
Let $X \sim \mathcal N(0,\mathbf{I}_d)$. Define the scalar random variable $Z := X^\top \beta^\star \sim \mathcal N(0,1)$. By Gaussian orthogonal decomposition, we may write
\[
X = Z \beta^\star + U,
\]
where $U \perp \beta^\star$, $U \sim \mathcal N(0, \mathbf{I}_d - \beta^\star\beta^{*\top})$,
and $Z$ and $U$ are independent.
Expanding the outer product yields
\[
XX^\top
= Z^2 \beta^\star\beta^{*\top}
+ Z(\beta^\star U^\top + U\beta^{*\top})
+ UU^\top.
\]
 Multiplying by $g(Z)$ and taking expectations,
the cross terms vanish by independence and symmetry:
\[
\mathbb E[g(Z)\, Z\, \beta^\star U^\top]
=
\mathbb E[g(Z)Z] \beta^*\;\mathbb E[U]^{T}
= 0,
\]
and similarly for its transpose. Hence,
\[
\mathbb E[g(Z)\, XX^\top]
=
\mathbb E[g(Z)\, Z^2]\, \beta^\star\beta^{*\top}
+
\mathbb E[g(Z)\, UU^\top].
\]
Since $U$ is independent of $Z$ and satisfies
$\mathbb E[UU^\top] = \mathbf{I}_d - \beta^\star\beta^{*\top}$, we obtain
\[
\mathbb E[g(Z)\, UU^\top]
=
\mathbb E[g(Z)](\mathbf{I}_d - \beta^\star\beta^{*\top}).
\]
Combining terms gives
\[
\mathbb E[g(Z)\, XX^\top]
=
\mathbb E[g(Z)] \mathbf{I}_d
+
\Big(
\mathbb E[Z^2 g(Z)]
-
\mathbb E[g(Z)]
\Big)\beta^\star\beta^{*\top},
\]
which establishes the desired decomposition.
\end{proof}

\begin{lemma}\label{lem:mixed-moment}
Let $X\sim\mathcal{N}(0,\mathbf{I}_d)$, $h\in\mathbb{R}^d$, and $v\in\mathbb{R}^d$ with $\norm{v}_2=1$. Then
\begin{equation}\label{eq:mixed-moment-exact}
\E\!\left[(X^\top h)^2(X^\top v)^4\right]
=
3\norm{h}_2^2+12\inner{h}{v}^2
\le 15\norm{h}_2^2.
\end{equation}
\end{lemma}

\begin{proof}
Let $c=\inner{h}{v}$ and write $h=cv+h_\perp$, where $h_\perp\perp v$. Setting $V=X^\top v$ and $W=X^\top h_\perp$, we have $X^\top h=cV+W$.
Moreover, $V\sim\mathcal{N}(0,1)$ and $W\sim\mathcal{N}(0,\norm{h}_2^2-c^2)$. Since $V$ and $W$ are jointly Gaussian and uncorrelated, they are independent. Therefore,
\[
\E\!\left[(X^\top h)^2(X^\top v)^4\right]
=\E\!\left[(cV+W)^2V^4\right] =c^2\E[V^6]+\E[W^2]\E[V^4] =15c^2+3\bigl(\norm{h}_2^2-c^2\bigr) =3\norm{h}_2^2+12\inner{h}{v}^2.
\]
The inequality follows from $\inner{h}{v}^2\le\norm{h}_2^2\norm{v}_2^2=\norm{h}_2^2$.
\end{proof}

\subsection{Proof of \cref{lem:local_geometry}}\label{subsec:proof_local_geo}
\lemConvexBasin*

\begin{proof}
Fix $\beta\in\Ball{\bstar}{R}$, and write $z:=X^\top\beta$, $z_\star:=X^\top\bstar$. 
Differentiating under the expectation and using $\E[Y\mid X]=f(z_\star)$ gives
\begin{equation}\label{eq:hessian-Q}
H(\beta)
=
\E\!\left[Q(z,z_\star)XX^\top\right],
\qquad
Q(z,z^*):=f'(z)^2+f''(z)\bigl(f(z)-f(z^*)\bigr).
\end{equation}

At $\beta=\bstar$, the residual term vanishes $(Q(z^*,z^*)=f'(z^*)^2+f''(z^*)\bigl(f(z^*)-f(z^*)\bigr)=f'(z^*)^2)$, so
\[
H(\bstar)
=
\E\!\left[f'(X^\top\bstar)^2XX^\top\right].
\]
Applying Lemma~\ref{lem:Gaussian_symmetry} with $g(s)=f'(s)^2$ yields
\[
H(\bstar)
=
a\Id+(b-a)\bstar(\bstar)^\top,
\]
where
\[
a:=\E[f'(Z)^2],
\qquad
b:=\E[Z^2f'(Z)^2],
\qquad
Z\sim\mathcal N(0,1).
\]
Thus $b$ is the eigenvalue in the direction $\bstar$, while $a$ is the eigenvalue on its orthogonal complement. Hence
\begin{equation}\label{eq:Hstar-bounds}
\mu\Id
\preceq
H(\bstar)
\preceq
\mu_1\Id.
\end{equation}

We next control the perturbation away from $\bstar$. Set
\[
h:=\beta-\bstar,
\qquad
\beta_t:=\bstar+th,
\qquad
z_t:=X^\top\beta_t=z_\star+tX^\top h,
\qquad t\in[0,1].
\]
By the Fundamental Theorem of Calculus,
\[
f'(z)^2-f'(z_\star)^2
=
(X^\top h)\int_0^1 2f'(z_t)f''(z_t)\,dt,
\]
and
\[
f(z)-f(z_\star)
=
(X^\top h)\int_0^1 f'(z_t)\,dt.
\]
Keeping the endpoint factor $f''(z)$ fixed in the second identity, we obtain
\begin{equation}\label{eq:Q-perturbation}
Q(z,z_\star)-Q(z_\star,z_\star)
=
(X^\top h)\overline C,
\end{equation}
where
\[
\overline C
:=
\int_0^1
f'(z_t)\bigl(2f''(z_t)+f''(z)\bigr)\,dt.
\]

Let $\norm{W}_{L^p}:=\left(\E\abs{W}^p\right)^{1/p}$ for any $p\geq 1$, and r.v. $W$. Since $\beta_t\in\Ball{\bstar}{R}$ for every $t\in[0,1]$, the definitions of the local moment envelopes give
\[
\norm{f'(z_t)}_{L^4}\le M_1(R),
\qquad
\norm{f''(z_t)}_{L^4}\le M_2(R),
\qquad
\norm{f''(z)}_{L^4}\le M_2(R).
\]
Therefore, by Fubini's theorem,  Minkowski's integral inequality and Hölder's inequality,
\begin{align}
\norm{\overline C}_{L^2}
&\le
\int_0^1
\left(
2\norm{f'(z_t)f''(z_t)}_{L^2}
+
\norm{f'(z_t)f''(z)}_{L^2}
\right)dt \notag\\
&\le
3M_1(R)M_2(R).
\label{eq:Cbar-bound}
\end{align}

Let $\Delta(\beta):=H(\beta)-H(\bstar)$. From
\eqref{eq:hessian-Q} and \eqref{eq:Q-perturbation},
\[
\Delta(\beta)
=
\E\!\left[(X^\top h)\overline C\,XX^\top\right].
\]
Since $\Delta(\beta)$ is symmetric,
\begin{align*}
\norm{\Delta(\beta)}_{\mathrm{op}}
&=
\sup_{\norm{v}_2=1}
\abs{
\E\!\left[(X^\top h)\overline C(X^\top v)^2\right]
}\\
&\le
\norm{\overline C}_{L^2}
\sup_{\norm{v}_2=1}
\left(
\E\!\left[(X^\top h)^2(X^\top v)^4\right]
\right)^{1/2}.
\end{align*}
Lemma~\ref{lem:mixed-moment} and \eqref{eq:Cbar-bound} imply
\begin{equation}\label{eq:hessian-perturbation-bound}
\norm{\Delta(\beta)}_{\mathrm{op}}
\le
3\sqrt{15}\,M_1(R)M_2(R)\norm{h}_2
\le
3\sqrt{15}\, M_1(R)M_2(R) R.
\end{equation}
By the radius condition,
\[
6\sqrt{15}\,R M_1(R)M_2(R)\le\mu,
\]
and hence
\begin{equation}\label{eq:Delta-mu-half}
\norm{\Delta(\beta)}_{\mathrm{op}}
\le
\frac{\mu}{2}.
\end{equation}

Combining \eqref{eq:Hstar-bounds} and \eqref{eq:Delta-mu-half} gives
\[
\frac{\mu}{2}\Id
\preceq
H(\beta)
\preceq
\left(\mu_1+\frac{\mu}{2}\right)\Id.
\]
Since $\Ball{\bstar}{R}$ is convex, $\mathcal L$ is $\mu/2$-strongly convex and $(\mu_1+\mu/2)$-smooth throughout the ball.
\end{proof}

\section{Sample splitting in \cref{alg:lineartimealgo}}
\begin{remark}
  Under the notation of \cref{alg:lineartimealgo}, suppose the dataset consists of \(N = C (P+1) d \log(d)\) samples for some absolute constant \(C>0\), among which exactly \(K=\varepsilon N\) samples are corrupted. The samples are partitioned uniformly at random, without replacement, into \(P+1\) subsets, each containing \(C d \log(d)\) samples. If \(d\log(d)\ge C_{1}\frac{\log(P+1)+\log(1/\delta)}{\varepsilon^{2}}\) and \(2C d\log(d)\ge \log(P+1)+\log(1/\delta)\), then with probability at least \(1-\delta\), every subset \(N_j\), for \(j=1,\ldots,P+1\), contains at most a \(2\varepsilon\) fraction of corrupted samples. Thus, without loss of generality, we may assume that after sample splitting in \cref{alg:lineartimealgo},
each bucket $N_1, N_2, \ldots, N_{P+1}$ contains at most a $2\varepsilon$ fraction
of corrupted samples.
 The same result appears as Claim~C.1 in \cite{das2026}, and we therefore omit the proof here, referring the reader to the proof of Claim~C.1 in \cite{das2026} for details.
  \end{remark}
\section{ Omitted Proofs of \cref{sec:lin}}\label{lin_proof}
\lemSpectralIni*
\begin{proof}
    Expanding $\expect{\Tilde{Y}\Tilde{Y}^T}$ we obtain
    \begin{align*}
        \expect{\Tilde{Y}\Tilde{Y}^T}&=\expect{(f({X^T\beta^\star})+\zeta)^2XX^T}\\
        &=\expect{(f({X^T\beta^\star}))^2 XX^T}+2\expect{f({X^T\beta^\star})XX^T \expect{\zeta| X}}+\expect{XX^T\expect{\zeta^2|X}}\\
        &=\expect{(f({X^T\beta^\star}))^2 XX^T}+\sigma^2\expect{XX^T}.\\
        &= \expect{(f({X^T\beta^\star}))^2 XX^T}+\sigma^2 \mathbf{I}_d
    \end{align*}
    The second and third equalities follow from the linearity of expectation together with $\expect{\zeta \mid X}=0$ and $\expect{\zeta^2 \mid X}=\sigma^2$, while the fourth equality follows from the assumption that $X \sim \mathcal{N}(0,\mathbf{I}_d)$ (see \cref{def:model}). Apply Stein's lemma (\cref{lem:steins}) on the first term and rearrange to obtain
    \begin{align}
        \expect{\Tilde{Y}\Tilde{Y}^T}
        &=(\sigma^2+\expect{(f({X^T\beta^\star}))^2})\mathbf{I}_d+2\expect{(f^{\prime}({X^T\beta^\star}))^2+f({X^T\beta^\star})\cdot f^{\prime\prime}({X^T\beta^\star})}\beta^\star \beta^{*\top}.
    \end{align}
    The claim now follows from \cref{assump:sq-curv}.
\end{proof}
\hypercontrativityofYX*
\begin{proof}
    We have to show that for all $v \in \mathbb{R}^{d}$, the following holds:
    \begin{equation}\label{eq: hyeprcontartive}
        \left( \mathbb{E}_{X \sim \mathcal{N}(0,\mathbf{I}_d)}
\big[ \langle YX , v \rangle^{4} \big] \right)^{\frac{1}{4}}
\;\le\;
C_4
\left( \mathbb{E}_{X \sim \mathcal{N}(0,\mathbf{I}_d)}
\big[ \langle YX, v \rangle^{2} \big] \right)^{\frac{1}{2}} .
    \end{equation}
Let's consider the left hand side of \cref{eq: hyeprcontartive}. 
\begin{align*}
    \mathbb{E}_{X \sim \mathcal{N}(0,\mathbf{I}_d)}
\big[ \langle YX , v \rangle^{4} \big]&=\mathbb{E}_{X \sim \mathcal{N}(0,\mathbf{I}_d)}
\big[\left( Y\langle X ,v \rangle \right)^{4} \big] =
\mathbb{E}[\,Y^{4}\,\langle X,v\rangle^{4}]\\
& \;\overset{(a)}{\leq}\;
\mathbb{E}[8\,( f\!\left(X^\top \beta^\star\right)^4+\zeta^4)\,\langle X,v\rangle^{4}]\\
&\overset{(b)}{=} 8\mathbb{E}[ f\!\left(X^\top \beta^\star\right)^4\,\langle X,v\rangle^{4}]+24 K_{4}^{4}\|v\|^4\\
&\overset{(c)}{\leq} 88\sqrt{\mathbb{E}[ f\!\left(X^\top \beta^\star\right)^8]}\|v\|^4+24 K_{4}^{4}\|v\|^4\\
&=\left(88\sqrt{\mathbb{E}[ f\!\left(X^\top \beta^\star\right)^8]}+24 K_{4}^{4}\right)\|v\|^4,
\end{align*}
where $(a)$ follows from $(a+b)^4\leq 8(a^4+b^4)$, $(b)$ follows from $X^{\top}v\sim \mathcal{N}(0,\|v\|^2)$ and $\mathbb{E}[\xi^4\mid X]=K_{4}^4$, and $(c)$ follows from Cauchy-Schwarz inequality and $X^{\top}v\sim \mathcal{N}(0,\|v\|^2)$. So, this implies 
\begin{equation}\label{eq:up}
    \mathbb{E}_{X \sim \mathcal{N}(0,\mathbf{I}_d)}
\big[ \langle YX , v \rangle^{4} \big]^{1/4}\leq \left(88\sqrt{\mathbb{E}[ f\!\left(X^\top \beta^\star\right)^8]}+24 K_{4}^{4}\right)^{1/4}\|v\|\overset{(a)}{\leq} 4\left(\mathbb{E}[ f\!\left(X^\top \beta^\star\right)^8]^{1/8}+ K_{4}\right)\|v\|,
\end{equation}
where $(a)$ follows from $(a+b)^{1/4}\leq a^{1/4}+b^{1/4}$. Now, let's consider the right hand side of \cref{eq: hyeprcontartive}. 
\begin{align*}
    \mathbb{E}_{X \sim \mathcal{N}(0,\mathbf{I}_d)}
\big[ \langle YX, v \rangle^{2} \big]&=\mathbb{E}_{X \sim \mathcal{N}(0,\mathbf{I}_d)}
\big[\left( Y\langle X ,v \rangle \right)^{2} \big] =
\mathbb{E}[\,Y^{2}\,\langle X,v\rangle^{2}]\\
& \overset{(a)}{=}
\mathbb{E}[\,( f\!\left(X^\top \beta^\star\right)^2+\zeta^2)\,\langle X,v\rangle^{2}]\\
&\overset{(b)}{=}\mathbb{E}[ f\!\left(X^\top \beta^\star\right)^2\,\langle X,v\rangle^{2}]+ \sigma^2\|v\|^2\geq\sigma^2\|v\|^2,
\end{align*}
where $(a)$ follows from $\mathbb{E}[\xi\mid X]=0$, $(b)$ follows from $\mathbb{E}[\xi^2\mid X]=\sigma^2$ and $X^{\top}v\sim \mathcal{N}(0,\|v\|^2)$. So, this implies 
\begin{equation}\label{eq:down}
    \left( \mathbb{E}_{X \sim \mathcal{N}(0,\mathbf{I}_d)}
\big[ \langle YX , v \rangle^{2} \big] \right)^{\frac{1}{2}}\geq \sigma\|v\|.
\end{equation}
Note that \cref{eq: hyeprcontartive} is satisfied for any positive $C_{4}$ when $v=0$. Now,
\begin{align*}
    \sup_{v\mathbb{R}^{d}:v\neq 0}\frac{ \left( \mathbb{E}_{X \sim \mathcal{N}(0,\mathbf{I}_d)}
\big[ \langle YX , v \rangle^{4} \big] \right)^{\frac{1}{4}}}{\left( \mathbb{E}_{X \sim \mathcal{N}(0,\mathbf{I}_d)}
\big[ \langle YX , v \rangle^{2} \big] \right)^{\frac{1}{2}}}
&\overset{(a)}{\leq} \sup_{v\mathbb{R}^{d}:v\neq 0}\frac{4\left(\mathbb{E}\left[ f\!\left(X^\top \beta^\star\right)^8\right]^{1/8}+ K_{4}\right)\|v\|}{\sigma \|v\|}\\
&\leq \frac{4\left(\mathbb{E}\left[ f\!\left(X^\top \beta^\star\right)^8\right]^{1/8}+ K_{4}\right)}{\sigma},
\end{align*}
where $(a)$ follows from \cref{eq:up} and \cref{eq:down}.
So, the above calculations suggest that one possible choice of $C_{4}$ is
\[C_{4}=\frac{4\left(\mathbb{E}\left[ f\!\left(X^\top \beta^\star\right)^8\right]^{1/8}+ K_{4}\right)}{\sigma}.\]
\end{proof}
\lintimespecinitalization*
\begin{proof}
 \Cref{lem:spectral} establishes that for the random vector
$\widetilde{Y} = YX$, the vector $\beta^\star$ is the leading eigenvector of
$\Sigma = \mathbb{E}[\widetilde{Y}\widetilde{Y}^\top]$. The corresponding
largest eigenvalue is
\[
\lambda_{\max}
=
\sigma^2
+
\mathbb{E}\!\left[f\!\left(X^\top \beta^\star\right)^2\right]
+
2\,\mathbb{E}\!\left[
\left(f'\!\left(X^\top \beta^\star\right)\right)^2
+
f\!\left(X^\top \beta^\star\right)
f''\!\left(X^\top \beta^\star\right)
\right]
=
\sigma^2
+
\mathbb{E}\!\left[f\!\left(X^\top \beta^\star\right)^2\right]
+ 2c .
\]
Moreover, all remaining eigenvalues are equal, each having value
\[
\sigma^2 + \mathbb{E}\!\left[f\!\left(X^\top \beta^\star\right)^2\right].
\]

\Cref{lem:hypercontracitivity of YX} shows that the random vector
$\widetilde{Y} = YX$ is $(4, C_{4})$-hypercontractive, where
\[
C_{4}
=
\frac{
4\left(
\mathbb{E}\!\left[ f\!\left(X^\top \beta^\star\right)^8 \right]^{1/8}
+ K_{4}
\right)
}{\sigma}.
\]

We now apply \Cref{thm:robust_hypercontractive_kePCA} to bound the distance
between $\hat{u}$ and the leading eigenvector $\beta^\star$ of $\Sigma$.
To do so, we verify that all assumptions of the theorem are satisfied.
In the notation of \Cref{thm:robust_hypercontractive_kePCA}, we have
\[
\rho
=
\Theta\!\left(
\frac{
16\left(
\mathbb{E}\!\left[ f\!\left(X^\top \beta^\star\right)^8 \right]^{1/8}
+ K_{4}
\right)^2
}{\sigma^2}
\, \epsilon^{1/2}
\right),
\]
and
\[
\vartheta
=
\frac{
4^6
\left(
\mathbb{E}\!\left[ f\!\left(X^\top \beta^\star\right)^8 \right]^{1/8}
+ K_{4}
\right)^6
}{
\sigma^6 \epsilon^{1/2}
}.
\]

To satisfy the sample complexity requirement of
\Cref{thm:robust_hypercontractive_kePCA}, we need
\[
m
=
\Theta\!\left(
\vartheta
\frac{d \log d + \log(1/\delta)}{\rho^2}
\right)
=
\Theta\!\left(
\frac{
16\left(
\mathbb{E}\!\left[ f\!\left(X^\top \beta^\star\right)^8 \right]^{1/8}
+ K_{4}
\right)^2
}{\sigma^2}
\frac{d \log d + \log(1/\delta)}{\epsilon^{3/2}}
\right).
\]

Furthermore, \Cref{thm:robust_hypercontractive_kePCA} requires $\rho \le 1$,
which is ensured by the assumption on the contamination level:
\[
\epsilon
=
O\!\left(
\frac{\sigma^4}{\left(
\mathbb{E}\!\left[ f\!\left(X^\top \beta^\star\right)^8 \right]^{1/8}
+ K_{4}
\right)^4}
\right)\leq 1/2.
\]

Thus, all assumptions of \Cref{thm:robust_hypercontractive_kePCA} are satisfied.
Consequently, with probability at least \(1 - \delta\), the output
\(\hat{u} \in \mathbb{R}^{d}\) of Algorithm~$\mathcal{A}_k$
from \cite{jambulapati2024black} satisfies
\[
\|\Sigma\|_{\mathrm{op}}
-
\operatorname{Tr}\!\left[\hat{u}^\top \Sigma \hat{u}\right]
=
O\!\left(
\rho \|\Sigma\|_{\mathrm{op}}
\right).
\]

Recall that the largest eigenvalue of $\Sigma$ is
$\sigma^2 + \mathbb{E}[f^2] + 2c$, while all remaining eigenvalues equal
$\sigma^2 + \mathbb{E}[f^2]$. Since these two quantities differ, we have
$\lambda_1 \neq \lambda_2$, where $\lambda_1$ and $\lambda_2$ denote the
largest and second largest eigenvalues of $\Sigma$, respectively.
Therefore,
\begin{align*}
\|\Sigma\|_{\mathrm{op}}
-
\operatorname{Tr}\!\left[\hat{u}^\top \Sigma \hat{u}\right]
&\ge
\lambda_1
-
\big(
\lambda_1 \langle \hat{u}, \beta^\star \rangle^2
+
(1 - \langle \hat{u}, \beta^\star \rangle^2)\lambda_2
\big) \\
&=
(\lambda_1 - \lambda_2)
\big(1 - \langle \hat{u}, \beta^\star \rangle^2\big)=2c \big(1 - \langle \hat{u}, \beta^\star \rangle^2\big) .\\
\end{align*}

Combining the above bounds yields
\begin{align*}
\big(1 - \langle \hat{u}, \beta^\star \rangle^2\big)
&=
O\!\left(
\rho \frac{\sigma^2 + \mathbb{E}[f^2] + c}{c}
\right) \\
&=
O\!\left(
\frac{
16\left(
\mathbb{E}\!\left[ f\!\left(X^\top \beta^\star\right)^8 \right]^{1/8}
+ K_{4}
\right)^2
\left(
\sigma^2 + \mathbb{E}[f^2] + c
\right)
}{\sigma^2 c}
\, \epsilon^{1/2}
\right),
\end{align*}
which further implies that 
\[\left|\langle \hat{u}, \beta^\star \rangle\right|=\sqrt{1-O\!\left(
\nicefrac{
16\left(
\mathbb{E}\!\left[ f\!\left(X^\top \beta^\star\right)^8 \right]^{1/8}
+ K_{4}
\right)^2
\left(
\sigma^2 + \mathbb{E}[f^2] + c
\right) \epsilon^{1/2}
}{\sigma^2 c}
\, 
\right)}.\]

 \begin{align*}
   \operatorname{dist}(\hat{u},\beta^\star)&=\min\{\|\hat{u}-\beta^\star\|_{2},\|\hat{u}+\beta^\star\|_{2}\}=\sqrt{2-2|\langle \hat{u}, \beta^\star \rangle|}\\
   &=\sqrt{2}\sqrt{1-\sqrt{1-O\!\left(
\frac{
16\left(
\mathbb{E}\!\left[ f\!\left(X^\top \beta^\star\right)^8 \right]^{1/8}
+ K_{4}
\right)^2
\left(
\sigma^2 + \mathbb{E}[f^2] + c
\right) \epsilon^{1/2}
}{\sigma^2 c}
\, 
\right)}}\\
&\overset{(a)}{\leq} 
\sqrt{2}\sqrt{
1 - \left(
1 - O\!\left(
\frac{
16\left(
\mathbb{E}\!\left[ f\!\left(X^\top \beta^\star\right)^8 \right]^{1/8}
+ K_4
\right)^2
\left(
\sigma^2 + \mathbb{E}\!\left[f\!\left(X^\top \beta^\star\right)^2\right] + c
\right)
\epsilon^{1/2}
}{
\sigma^2 c
}
\right)
\right)
}\\
&=O\!\left(
\frac{
4\left(
\mathbb{E}\!\left[ f\!\left(X^\top \beta^\star\right)^8 \right]^{1/8}
+ K_4
\right)
\left(
\sigma^2 + \mathbb{E}\!\left[f\!\left(X^\top \beta^\star\right)^2\right] + c
\right)^{1/2}
\epsilon^{1/4}
}{
\sigma \sqrt{c}
}\right)\\&=O\!\left(
\frac{
C_{4}
\left(
\sigma^2 + \mathbb{E}\!\left[f\!\left(X^\top \beta^\star\right)^2\right] + c
\right)^{1/2}
\epsilon^{1/4}
}{ \sqrt{c}
}\right),
 \end{align*}
where $(a)$ follow from assumption on the contamination level, namely  $\nicefrac{
C_{4}^4
\left(
\sigma^2 + \mathbb{E}\!\left[f\!\left(X^\top \beta^\star\right)^2\right] + c
\right)^{2}
\epsilon
}{ c^2
}=O(1)\leq 1/2,$  which implies $\nicefrac{
C_{4}^2
\left(
\sigma^2 + \mathbb{E}\!\left[f\!\left(X^\top \beta^\star\right)^2\right] + c
\right)
\epsilon^{1/2}
}{ c
}=O(1)\leq 1$, together with elementary inequality  $\sqrt{1-x}\geq 1-x$ when $0\leq x\leq 1$. 
Since we set $\beta_0 = \hat{u}$, it follows from the above that
\[
\operatorname{dist}(\beta_0, \beta^\star) =O\left(
\frac{
C_4 \left(
\sigma^2 + \mathbb{E}\!\left[f\!\left(X^\top \beta^\star\right)^2\right] + c
\right)^{1/2}
}{
\sqrt{c}
}
\epsilon^{1/4}\right).
\]

The stated running time follows directly from
\Cref{thm:robust_hypercontractive_kePCA}.
\end{proof}
\linRobustGradientDescnt*
First, we present two key lemmas needed for the proof of \cref{theorem: li-time-grads-descent}, followed by the proof itself. The first key lemma expresses the distance of the iterate at time \( t+1 \) from \( \beta^\star \) in terms of the distance of the iterate at time \( t \) from \( \beta^\star \). The relation is as follows:
\begin{lemma}\label{lemma:relation of current iterate to previous one}
    Suppose $\beta_t$ obeys $\left\|\beta_t-\beta^\star\right\| \leq R$, $\eta \leq 2 /(\alpha+\gamma)$ and $g_t$ be the valid robust gradient estimator of population risk at $x_t$ (see \cref{definition:rodust estimator of gradient}). Then, with probability at least $1-\delta$, the new iterate $\beta_{t+1}$ obtained according to $\beta_{t+1} = \beta_t - \eta g_t$ satisfies
    \begin{equation}\label{eq:Guarantees enjoyed by new iterate2}
 \left\|\beta_{t+1}-\beta^\star\right\| \leq\left(\sqrt{1-\frac{2 \eta \alpha \gamma}{\alpha+\gamma}}+\eta A(m, \delta, \epsilon)\right)\left\|\beta_t-\beta^\star\right\|+\eta B(m, \delta, \epsilon)
    \end{equation}
\end{lemma}
The above lemma and its proof are standard in robust statistics literature; for example, see Theorem 1 in \cite{prasad2020robust} and Lemma 2.2 from \cite{bunaR24_robust_phase}. Look at any of the above-mentioned references for the proof of the \cref{lemma:relation of current iterate to previous one}. We are omitting the proof here. Our second key lemma computes the trace and operator norm of the covariance matrix of the gradient. 

\begin{lemma}\label{lemma: calculations of trace and operator norm of gradients}
   Let's assume that $\beta, \beta^\star \in \mathbb{R}^d$ are fixed vectors and $X \sim \mathcal{N}(0, \mathbf{I}_d)$, $y = f(X^{\top}\beta^\star)+\zeta$. Let $\Sigma = \operatorname{Var}\left( \left(f(X^{\top}\beta) - y\right)f'(X^{\top}\beta) X \right)$ be the variance of the loss gradient. Let's define $\phi_{1}:=\sup_{\beta\in \mathcal{B}(\beta^\star, R)}(\mathbb{E}\!\left[f'(X^{\top}\beta)^{16})\right]^{1/4}$ and $\phi_{2}:=\sup_{\beta\in \mathcal{B}(\beta^\star, R)}(\mathbb{E}\!\left[f'(X^{\top}\beta)^{4}\right]^{1/2}$. If $\|\beta - \beta^\star\| \leq R$, then
    \[
    \begin{aligned}
\|\Sigma\|_{\mathrm{op}} &\leq 6 \phi_{1}\|\beta-\beta^\star\|^2+ \sqrt{3} \sigma^2  \phi_{2}.
    \end{aligned}
    \]
\end{lemma}
\begin{proof}
\begin{equation}\label{eq:expre_Sigma}
    \Sigma
=
\mathbb{E}\!\left[\left(f(X^{\top} \beta) - y\right)^2 f'(X^{\top} \beta)^2 \, X X^{\top}\right]
-
\mathbb{E}\!\left[\left(f(X^{\top} \beta) - y\right)f'(X^{\top} \beta)\, X\right]
\mathbb{E}\!\left[\left(f(X^{\top} \beta) - y\right)f'(X^{\top} \beta)\, X\right]^{\top}
\end{equation}
We now treat each term on the right hand side of \cref{eq:expre_Sigma} separately. Consider the first term $\mathbb{E}\!\left[\left(f(X^{\top} \beta) - y\right)^2 f'(X^{\top} \beta)^2 \, X X^{\top}\right]$.
 \begin{align*}
\mathbb{E}\!\left[\left(f(X^{\top}\beta)-y\right)^2 f'(X^{\top}\beta)^2\, X X^{\top}\right]
&=
\mathbb{E}\!\left[\left(f(X^{\top}\beta)-f(X^{\top}\beta^\star)\right)^2
f'(X^{\top}\beta)^2\, X X^{\top}\right]
+
\mathbb{E}\!\left[(\zeta)^2 
f'(X^{\top}\beta)^2\, X X^{\top}\right] \\
&\quad
-
2\mathbb{E}\!\left[\left(f(X^{\top}\beta)-f(X^{\top}\beta^\star)\right)
\zeta\, \, f'(X^{\top}\beta)\, X X^{\top}\right] \\
&\overset{(a)}{=}
\mathbb{E}\!\left[\left(f(X^{\top}\beta)-f(X^{\top}\beta^\star)\right)^2
f'(X^{\top}\beta)^2\, X X^{\top}\right]+
\sigma^2\,
\mathbb{E}\!\left[
f'(X^{\top}\beta)^2\, X X^{\top}\right],
\end{align*}
where $(a)$ follows from $\mathbb{E}[\xi\mid X]=0$ and $\mathbb{E}[\xi^2\mid X]=\sigma^2$. Now, consider the second term on the right hand side of \cref{eq:expre_Sigma}.
\begin{align*}
   \mu(f,\beta^\star,\beta):= \mathbb{E}\!\left[\left(f(X^{\top} \beta) - y\right)f'(X^{\top} \beta)\, X\right]
&\overset{(a)}{=}\mathbb{E}\!\left[\left(f(X^{\top} \beta) - f(X^{\top} \beta^\star)\right)f'(X^{\top} \beta)\, X\right],
\end{align*}
where $(a)$ follows from $\mathbb{E}[\xi\mid X]=0$. So, this implies 
\begin{align*}
 \Sigma+\mu(f,\beta^\star,\beta)\mu(f,\beta^\star,\beta)^{\top} &=\Sigma+\mathbb{E}\!\left[\left(f(X^{\top} \beta) - f(X^{\top} \beta^\star)\right)f'(X^{\top} \beta)\, X\right]\mathbb{E}\!\left[\left(f(X^{\top} \beta) - f(X^{\top} \beta^\star)\right)f'(X^{\top} \beta)\, X\right]^{\top}\\
 &=\mathbb{E}\!\left[\left(f(X^{\top}\beta)-f(X^{\top}\beta^\star)\right)^2
f'(X^{\top}\beta)^2\, X X^{\top}\right]+
\sigma^2\,
\mathbb{E}\!\left[
f'(X^{\top}\beta)^2\, X X^{\top}\right].
\end{align*}
As both $\Sigma$ and $\mu(f,\beta^\star,\beta)\mu(f,\beta^\star,\beta)^{\top}$ are positive definite matrices, hence
\begin{equation}\label{eq:positive_definite_sense}
    \Sigma\preceq \mathbb{E}\!\left[\left(f(X^{\top}\beta)-f(X^{\top}\beta^\star)\right)^2
f'(X^{\top}\beta)^2\, X X^{\top}\right]+
\sigma^2\,
\mathbb{E}\!\left[
f'(X^{\top}\beta)^2\, X X^{\top}\right].
\end{equation}
\cref{eq:positive_definite_sense} implies
\begin{equation}\label{eq:operator_norm_bound}
    \|\Sigma\|_{\mathrm{op}}\leq \left\|\mathbb{E}\!\left[\left(f(X^{\top}\beta)-f(X^{\top}\beta^\star)\right)^2
f'(X^{\top}\beta)^2\, X X^{\top}\right]\right\|_{\mathrm{op}}+\sigma^2\left\|\mathbb{E}\!\left[
f'(X^{\top}\beta)^2\, X X^{\top}\right]\right\|_{\mathrm{op}}.
\end{equation}
We now treat each term on the right hand side of \cref{eq:operator_norm_bound} separately. 

Consider the first term
\[
\left\|\mathbb{E}\!\left[\left(f(X^{\top}\beta)-f(X^{\top}\beta^\star)\right)^2
f'(X^{\top}\beta)^2\,XX^{\top}\right]\right\|_{\mathrm{op}}.
\]
Expanding the operator norm, we obtain
\[
\sup_{v:\|v\|=1}
v^\top
\mathbb{E}\!\left[
\left(f(X^{\top}\beta)-f(X^{\top}\beta^\star)\right)^2
f'(X^{\top}\beta)^2
XX^\top
\right]v
=
\sup_{v:\|v\|=1}
\mathbb{E}\!\left[
\left(f(X^{\top}\beta)-f(X^{\top}\beta^\star)\right)^2
f'(X^{\top}\beta)^2
(X^\top v)^2
\right].
\]

Let
\[
h:=\beta-\beta^\star,\qquad
\beta_t:=\beta^\star+th,\qquad
z_t:=X^\top\beta_t=X^\top\beta^\star+tX^\top h,\qquad
t\in[0,1].
\]
By the Fundamental Theorem of Calculus,
\[
f(X^\top\beta)-f(X^\top\beta^\star)
=
(X^\top h)\int_0^1f'(z_t)\,dt.
\]
Substituting this identity into the previous display gives
\begin{align*}
&\sup_{v:\|v\|=1}
\mathbb{E}\!\left[
\left(f(X^\top\beta)-f(X^\top\beta^\star)\right)^2
f'(X^\top\beta)^2
(X^\top v)^2
\right]\\
&=
\sup_{v:\|v\|=1}
\mathbb{E}\!\left[
(X^\top h)^2
\left(\int_0^1f'(z_t)\,dt\right)^2
f'(X^\top\beta)^2
(X^\top v)^2
\right]\\
&\le
\sup_{v:\|v\|=1}
\int_0^1
\mathbb{E}\!\left[
(X^\top h)^2
f'(z_t)^2
f'(X^\top\beta)^2
(X^\top v)^2
\right]dt,
\end{align*}
where the inequality follows from Jensen's inequality together with Fubini's theorem, since the integrand is non-negative.

Next, applying the Cauchy--Schwarz inequality yields
\begin{align*}
&\sup_{v:\|v\|=1}
\int_0^1
\mathbb{E}\!\left[
(X^\top h)^2
f'(z_t)^2
f'(X^\top\beta)^2
(X^\top v)^2
\right]dt\\
&\le
\sup_{v:\|v\|=1}
\int_0^1
\left(
\mathbb{E}\!\left[
(X^\top h)^4
f'(z_t)^4
f'(X^\top\beta)^4
\right]
\right)^{1/2}
\left(
\mathbb{E}\!\left[
(X^\top v)^4
\right]
\right)^{1/2}
dt\\
&=
\sqrt{3}
\int_0^1
\left(
\mathbb{E}\!\left[
(X^\top h)^4
f'(z_t)^4
f'(X^\top\beta)^4
\right]
\right)^{1/2}
dt,
\end{align*}
where the last equality follows from the fact that $X^\top v\sim\mathcal{N}(0,1)$ for every $\|v\|=1$, and hence $\mathbb{E}[(X^\top v)^4]=3$.

We now bound the remaining expectation using Hölder's inequality with exponents $4$, $4$, and $2$:
\begin{align*}
&\sqrt{3}
\int_0^1
\left(
\mathbb{E}\!\left[
(X^\top h)^4
f'(z_t)^4
f'(X^\top\beta)^4
\right]
\right)^{1/2}
dt\\
&\le
\sqrt{3}
\int_0^1
\mathbb{E}\!\left[f'(z_t)^{16}\right]^{1/8}
\mathbb{E}\!\left[f'(X^\top\beta)^{16}\right]^{1/8}
\mathbb{E}\!\left[(X^\top h)^8\right]^{1/4}
\,dt.
\end{align*}

Since
\[
\beta_t=\beta^\star+t(\beta-\beta^\star)\in\mathcal{B}(\beta^\star,R),
\qquad t\in[0,1],
\]
we have
\[
\mathbb{E}\!\left[f'(z_t)^{16}\right]^{1/8}
\le
\sup_{\beta\in\mathcal{B}(\beta^\star,R)}
\mathbb{E}\!\left[f'(X^\top\beta)^{16}\right]^{1/8}.
\]
Substituting this bound into the previous display, we obtain
\begin{align*}
&\sqrt{3}
\int_0^1
\mathbb{E}\!\left[f'(z_t)^{16}\right]^{1/8}
\mathbb{E}\!\left[f'(X^\top\beta)^{16}\right]^{1/8}
\mathbb{E}\!\left[(X^\top h)^8\right]^{1/4}
\,dt\\
&\le
\sqrt{3}
\int_0^1
\left(
\sup_{\beta\in\mathcal{B}(\beta^\star,R)}
\mathbb{E}\!\left[f'(X^\top\beta)^{16}\right]^{1/8}
\right)^2
\mathbb{E}\!\left[(X^\top h)^8\right]^{1/4}
\,dt\\
&=
\sqrt{3}
\left(
\sup_{\beta\in\mathcal{B}(\beta^\star,R)}
\mathbb{E}\!\left[f'(X^\top\beta)^{16}\right]^{1/8}
\right)^2
\mathbb{E}\!\left[(X^\top h)^8\right]^{1/4},
\end{align*}
where the final equality follows from the identity $\int_0^1dt=1$.

Finally, observe that
\[
X^\top h\sim\mathcal{N}(0,\|h\|_2^2),
\]
which implies
\[
\mathbb{E}\!\left[(X^\top h)^8\right]
=
105\|h\|_2^8.
\]
Consequently,
\begin{align*}
&\sqrt{3}
\left(
\sup_{\beta\in\mathcal{B}(\beta^\star,R)}
\mathbb{E}\!\left[f'(X^\top\beta)^{16}\right]^{1/8}
\right)^2
\mathbb{E}\!\left[(X^\top h)^8\right]^{1/4}\\
&=
\sqrt{3}(105)^{1/4}
\sup_{\beta\in\mathcal{B}(\beta^\star,R)}
\mathbb{E}\!\left[f'(X^\top\beta)^{16}\right]^{1/4}
\|h\|_2^2\\
&\le
6
\sup_{\beta\in\mathcal{B}(\beta^\star,R)}
\mathbb{E}\!\left[f'(X^\top\beta)^{16}\right]^{1/4}
\|\beta-\beta^\star\|_2^2,
\end{align*}
where the last inequality follows from the fact that
$\sqrt{3}(105)^{1/4}\le 6$, and the equality
\[
\left(
\sup_{\beta\in\mathcal{B}(\beta^\star,R)}
\mathbb{E}\!\left[f'(X^\top\beta)^{16}\right]^{1/8}
\right)^2
=
\sup_{\beta\in\mathcal{B}(\beta^\star,R)}
\mathbb{E}\!\left[f'(X^\top\beta)^{16}\right]^{1/4}
\]
holds since the function $x\mapsto x^2$ is increasing on $[0,\infty)$.

Now, consider the second term $\left\|\mathbb{E}\!\left[
f'(X^{\top}\beta)^2\, X X^{\top}\right]\right\|_{\mathrm{op}}$.
\begin{align*}
    \left\|\mathbb{E}\!\left[
f'(X^{\top}\beta)^2\, X X^{\top}\right]\right\|_{\mathrm{op}}&=\sup_{v:\|v\|=1}v\top\mathbb{E}\!\left[f'(X^{\top}\beta)^2\, X X^{\top}\right]v=\sup_{v:\|v\|=1}\mathbb{E}\!\left[f'(X^{\top}\beta)^2\,  (X^{\top}v)^2\right]\\
&\overset{(a)}{\leq} \sqrt{3} \sup_{\beta\in \mathcal{B}(\beta^\star, R)}(\mathbb{E}\!\left[f'(X^{\top}\beta)^{4})\right]^{1/2},
\end{align*}
where $(a)$ follows from the Cauchy-Schwarz inequality and the trivial upper bound
that for all $\beta \in \mathcal{B}(\beta^\star, R)$,
$(\mathbb{E}\!\left[f'(X^{\top}\beta)^{4})\right]^{1/2}\leq \sup_{\beta\in \mathcal{B}(\beta^\star, R)}(\mathbb{E}\!\left[f'(X^{\top}\beta)^{4})\right]^{1/2}$. Now,
\begin{align*}
    \|\Sigma\|_{\mathrm{op}}
&\leq \left\|\mathbb{E}\!\left[\left(f(X^{\top}\beta)-f(X^{\top}\beta^\star)\right)^2
f'(X^{\top}\beta)^2\, X X^{\top}\right]\right\|+\left\|
\sigma^2\,
\mathbb{E}\!\left[
f'(X^{\top}\beta)^2\, X X^{\top}\right]\right\|\\
&\leq 6 \sup_{\beta\in \mathcal{B}(\beta^\star, R)}(\mathbb{E}\!\left[f'(X^{\top}\beta)^{16})\right]^{1/4}\|\beta-\beta^\star\|^2+ \sigma^2 \sqrt{3} \sup_{\beta\in \mathcal{B}(\beta^\star, R)}(\mathbb{E}\!\left[f'(X^{\top}\beta)^{4})\right]^{1/2}\\
&= 6 \phi_{1}\|\beta-\beta^\star\|^2+ \sigma^2 \sqrt{3} \phi_{2}.
\end{align*}
\end{proof}

Now we go for the proof for \cref{theorem: li-time-grads-descent}. The proof is an adaptation of the proof of the guarantees of the robust gradient descent algorithm (particularly theorem 3.3) from \cite{bunaR24_robust_phase}. The calculations will vary, but the overall flow of arguments remains the same.
\begin{proof}[Proof of \cref{theorem: li-time-grads-descent}]
Like \cite{bunaR24_robust_phase}, first, we try to prove by induction that all iterates $\left(\beta_t\right)_{t=0}^{P-1}$ lie inside the ball centered at $\beta^\star$ with radius $R$, because if we can show this then we can write $\|\beta_{t+1}-\beta^\star\|$ in terms of $\|\beta_{t}-\beta^\star\|$ using \cref{lemma:relation of current iterate to previous one}.

\textbf{Induction argument}: To avoid redundancy, we consider the case the first case when $\operatorname{dist}\left(\hat{u}, \beta^\star\right)=\left\|\Hat{u}-\beta^\star\right\|$ (otherwise, repeat the same proof with $-\beta^\star$ ). Note that for the $-\beta^\star$  case, the proof remains the same.  The $n=0$ case follows from the assumptions of the theorem. So, for $n=0$ case, $\|\hat{u}-\beta^\star\|\leq R$. Let's assume that the induction hypothesis is true till some $t \in\{0,1, \ldots, P-1\}, $ which implies $ \left\|\beta_t-\beta^\star\right\| \leq R$. Now, our goal is to show that $\left\|\beta_{t+1}-\beta^\star\right\| \leq R.$  Let us recall the definition of a robust estimator of the gradient at the point $x_t$ (see \cref{definition:rodust estimator of gradient}). It states that $g(\beta_t,T,\delta,\epsilon)$ is a robust gradient estimator of population risk at $\beta_t$ if there exist two functions $
A, B : \mathbb{N} \times [0,1]^2 \rightarrow \mathbb{R}
$
such that, with probability at least $1 - \delta$, the following bound holds:
\[
\left\|g(\beta_t,T,\delta,\epsilon)  - \nabla r(\beta_t) \right\| \leq A(\tilde{m}, \delta, \epsilon) \cdot \|\beta_t - \beta^\star\| + B(\tilde{m}, \delta, \epsilon).
\]
 Let's define $g_t :=g(\beta_t,T,\delta,\epsilon)$.  According to \cref{thm:lin-tim-robut-Mean Value Theorem}, we know that, with probability at least $1-\delta$,
\begin{equation}\label{eq: error bounf for t th iterate}
  \left\|g_t-\nabla r\left(\beta_t\right)\right\|=O\left(\sqrt{\|\Sigma\|_{\text {op }} \epsilon}\right),  
\end{equation}
where $\Sigma=\operatorname{Var}\left( \left(f(X^{\top}\beta) - y\right)f'(X^{\top}\beta)X\right)$ and  $\beta_t$ is treated as a fixed vector. Note that since a fresh sample is used to compute each $g_t$ at each time $t$, we have the following high-probability bound:
\begin{equation} \label{eq:conditional expectation}
\mathbb{P} \left( 
\left\| g_t - \nabla r(\beta_t) \right\| = O\left( 
 \sqrt{ \|\Sigma\|_{\mathrm{op}} \, \epsilon } 
\right) \, \Big| \, \beta_t \right) \geq 1 - \delta.
\end{equation}
Now, taking expectations with respect to $\beta_{t}$ on both sides of the above equation \ref{eq:conditional expectation} gives
$$\mathbb{P}\left(\left\|g_t-\nabla r\left(\beta_t\right)\right\|=O\left(\sqrt{\|\Sigma\|_{\mathrm{op}} \epsilon}\right)\right)\geq 1-\delta.$$

Using \cref{lemma: calculations of trace and operator norm of gradients}, we can say that
$\|\Sigma\|_{\mathrm{op}}\leq 6 \phi_{1}\|\beta_{t}-\beta^\star\|^2+ \sqrt{3} \sigma^2  \phi_{2}$. Now, putting the above bounds and using inequality $\sqrt{a+b} \leq \sqrt{a}+\sqrt{b}$, we get that, with probability at least $1-\delta$,
$$
\left\|g_t-\nabla r\left(\beta_t\right)\right\| \leq A(\tilde{m}, \delta, \epsilon)\left\|x_t-x^*\right\|+B(\tilde{m}, \delta, \epsilon)
$$

where $A(\tilde{m}, \delta, \epsilon)$ and $B(\tilde{m}, \delta, \epsilon)$ are defined as \begin{align*}
 A(\bar{m}, \delta, \epsilon):=O\left(\sqrt{ \phi_{1}}\sqrt{\epsilon}\right), \quad \text{and} \quad B(\bar{m}, \delta, \epsilon):=O\left(\sigma \sqrt{\phi_{2}} \sqrt{\epsilon}\right).
\end{align*}
So, $g_{t}$ is a valid robust gradient estimator of population risk at $\beta_t$. So, we have validated all assumptions of \cref{lemma:relation of current iterate to previous one}. Using \cref{lemma:relation of current iterate to previous one} we can say that the, with probability at least $1-\delta$,
\begin{equation}\label{eq:iniisde iterate proof 1}
   \left\|\beta_{t+1}-\beta^\star\right\| \leq\left(\sqrt{1-\frac{2 \eta \alpha \gamma}{\alpha+\gamma}}+\eta A(\tilde{m}, \delta, \epsilon)\right)\left\|\beta_t-\beta^\star\right\|+\eta B(\tilde{m}, \delta, \epsilon), 
\end{equation}
Now, our goal is to show that the right side of \cref{eq:iniisde iterate proof 1} is at most $R$. For this,  we show that $\sqrt{1-2 \eta \alpha \gamma /(\alpha+\gamma)}=\frac{\alpha-\gamma}{\alpha+\gamma}<1, \eta A(\tilde{m}, \delta, \epsilon) \leq \frac{\gamma}{\alpha+\gamma}$, and $\eta B(\tilde{m}, \delta, \epsilon) \leq \frac{\gamma R}{\alpha+\gamma}$ for the chosen values of $\tilde{m}$ and $\epsilon$. Then, using the induction hypothesis $\|\beta_t - \beta^\star\|\leq R$ and the above bounds, we conclude that with probability at least $1 - \delta$, the following holds:
\begin{equation}\label{eq:Gurantees of x_t needed to conclude final statements}
    \|\beta_{t+1} - \beta^\star\| \leq R.
\end{equation}
Our induction argument ends here. We now establish above bounds one by one.
\begin{itemize}
    \item  $\eta A(\bar{m}, \delta, \epsilon) \leq \gamma/\alpha+\gamma$ : Note that,

$$
\eta A(\tilde{m}, \delta, \epsilon) \leq \frac{2}{\alpha+\gamma} A(\tilde{m}, \delta, \epsilon)=\frac{\sqrt{\phi_{1}}}{\alpha+\gamma}O\left(\sqrt{\epsilon}\right) .
$$

If $\epsilon$ is chosen to be a small constant, such that $\epsilon\leq \frac{C_{2}\gamma^2}{\phi_{1}}$, then one can adjust the values of $\epsilon$ in such a way that the right-hand side of the above equation is at most $\gamma/(\alpha+\gamma)$.

\item $\eta B(\tilde{m}, \delta, \epsilon) \leq \frac{\gamma R}{\alpha+\gamma}$ : Note that, 
\begin{align*}
   \eta B(\tilde{m}, \delta, \epsilon)&=O\left(\frac{\sigma \sqrt{\phi_{2}}}{\alpha+\gamma}\left(\sqrt{\epsilon}\right)\right)=\frac{\sigma \sqrt{\phi_{2}}}{\alpha+\gamma}O\left(\sqrt{\epsilon}\right).
\end{align*}
As shown previously, if we choose $\epsilon\leq \frac{C_{3}\gamma^2 R^2}{\sigma^2\phi_{2}}$ then we can adjust the values of $\epsilon$ in a such a way that the right-hand side of the above equation is at most $\frac{\gamma R}{\alpha+\gamma}$.
\end{itemize}
Note that the above part is similar to the ``The induction" section of the proof of Theorem 3.3 from \cite{bunaR24_robust_phase}.

\textbf{Guarantees for $\beta_P$}: We have shown that with probability at least $1-P \delta$, \cref{eq:iniisde iterate proof 1} and \cref{eq:Gurantees of x_t needed to conclude final statements} hold for every $t \in\{0,1, \ldots, P-1\}$. So, we can conclude that for every $t \in\{0,1, \ldots, P-1\}$ :
\begin{align*}
    \left\|\beta_{t+1}-\beta^\star\right\|& \leq\left(\sqrt{1-\frac{2 \eta \alpha \gamma}{\alpha+\gamma}}+\eta A(\tilde{m}, \delta, \epsilon)\right)\left\|\beta_t-\beta^\star\right\|+\eta B(\tilde{m}, \delta, \epsilon)\\
    &\leq \left(\frac{\alpha}{\alpha+\gamma}\right)\left\|\beta_t-\beta^\star\right\|+\eta B(\tilde{m}, \delta, \epsilon).
\end{align*}
Iterating this over $t \in\{0,1, \ldots, P-1\}$ and using $\left\|\hat{u}-\beta^\star\right\| \leq R$, we have
\begin{align*}
\left\|\beta_P-\beta^\star\right\|&\overset{(a)}{\leq}\left(1-\frac{\gamma}{\alpha+\gamma}\right)^P\left\|\hat{u}-\beta^\star\right\|+\frac{\eta B(\tilde{m}, \delta, \epsilon)}{1-\left(\frac{\alpha}{\alpha+\gamma}\right)}\\
& \leq R \exp \left(- P\frac{ \gamma}{\alpha+\gamma}\right)+\frac{B(\tilde{m}, \delta, \epsilon)}{\frac{ \gamma}{\alpha+\gamma}},
\end{align*}
 The first and second terms of the right-hand side of step $(a)$ follow from the iteration argument and the sum of the infinite geometric series, respectively. Consider the second term. We know that $\epsilon\leq 1/2$ and this further implies
\[\frac{B(\tilde{m}, \delta, \epsilon)}{\frac{ \gamma}{\alpha+\gamma}}=\frac{\frac{\sigma \sqrt{\phi_{2}}}{\alpha+\gamma}O\left(\sqrt{\epsilon}\right)}{\frac{ \gamma}{\alpha+\gamma}}=\frac{\sigma \sqrt{\phi_{2}}}{\gamma}O\left(\sqrt{\epsilon}\right) .\]
By combining the two bounds mentioned above, we obtain
\begin{equation}\label{eq19}
   \left\|\beta_P-\beta^\star\right\|\leq R \exp \left(- P\frac{ \gamma}{\alpha+\gamma}\right)+ \frac{\sigma \sqrt{\phi_{2}}}{\gamma}O\left(\sqrt{\epsilon}\right).
\end{equation}
This implies
\begin{align*}
\left|\|\beta_{P}\|-1\right|&=\left|\|\beta_{P}\|-\|\beta^*\|\right|\le\|\beta_P-\beta^\star\|\\ 
    &\leq R \exp \left(- P\frac{ \gamma}{\alpha+\gamma}\right)+ \frac{\sigma \sqrt{\phi_{2}}}{\gamma}O\left(\sqrt{\epsilon}\right).
\end{align*}
Now,
\begin{align*}
    \left\|\frac{\beta_{P}}{\|\beta_{P}\|}-\beta^\star\right\|&=\left\|\frac{\beta_{P}}{\|\beta_{P}\|}-\beta_{P}+\beta_{P}-\beta^\star\right\|\leq \left\|\frac{\beta_{P}}{\|\beta_{P}\|}-\beta_{P}\right\|+\|\beta_{P}-\beta^\star\|\\
&=\left|\|\beta_{P}\|-1\right|+\|\beta_{P}-\beta^\star\|\\
    &\leq 2R \exp \left(- P\frac{ \gamma}{\alpha+\gamma}\right)+ \frac{2\sigma \sqrt{\phi_{2}}}{\gamma}O\left(\sqrt{\epsilon}\right).
\end{align*}
The time complexity of \cref{theorem: li-time-grads-descent} follows from time complexity of \cref{thm:lin-tim-robut-Mean Value Theorem}.
\end{proof}
\section{Numericals}\label{sec:numericals}

\begin{table}[htbp]
    \centering
    \small
    \setlength{\tabcolsep}{3pt}
    \caption{Explicit parameter values for \cref{thm:lintime}.}
    \resizebox{\textwidth}{!}{%
    \begin{tabular}{lccccccc}
        \toprule
        \textbf{Function}
        & \textbf{ESC}
        & $\boldsymbol{\mu}$
        & $\boldsymbol{\mu_1}$
        & $\boldsymbol{R}$
        & $\boldsymbol{C_{lip}(R)}$
        & $\boldsymbol{\phi_1}$
        & $\boldsymbol{\phi_2}$ \\
        \midrule
        Logistic/Sigmoid
        & $3.12 \times 10^{-2}$
        & $2.37 \times 10^{-2}$
        & $4.48 \times 10^{-2}$
        & $6.00 \times 10^{-2}$
        & $1.69 \times 10^{-2}$
        & $3.01 \times 10^{-3}$
        & $4.88 \times 10^{-2}$ \\

        Tanh
        & $1.82 \times 10^{-1}$
        & $1.08 \times 10^{-1}$
        & $4.64 \times 10^{-1}$
        & $9.50 \times 10^{-3}$
        & $4.63 \times 10^{-1}$
        & $6.49 \times 10^{-1}$
        & $5.87 \times 10^{-1}$ \\

        Probit
        & $4.59 \times 10^{-2}$
        & $3.06 \times 10^{-2}$
        & $9.19 \times 10^{-2}$
        & $2.00 \times 10^{-2}$
        & $6.29 \times 10^{-2}$
        & $1.79 \times 10^{-2}$
        & $1.07 \times 10^{-1}$ \\

        Phase Retrieval
        & $6.00 \times 10^{0}$
        & $4.00 \times 10^{0}$
        & $1.20 \times 10^{1}$
        & $3.10 \times 10^{-2}$
        & $5.43 \times 10^{0}$
        & $6.82 \times 10^{2}$
        & $7.36 \times 10^{0}$ \\

        GeLU
        & $4.86 \times 10^{-1}$
        & $4.56 \times 10^{-1}$
        & $5.78 \times 10^{-1}$
        & $3.90 \times 10^{-2}$
        & $4.96 \times 10^{-1}$
        & $1.01 \times 10^{0}$
        & $6.65 \times 10^{-1}$ \\

        Swish
        & $4.17 \times 10^{-1}$
        & $3.79 \times 10^{-1}$
        & $5.17 \times 10^{-1}$
        & $5.30 \times 10^{-2}$
        & $3.06 \times 10^{-1}$
        & $7.75 \times 10^{-1}$
        & $5.48 \times 10^{-1}$ \\
        \bottomrule
    \end{tabular}%
    }
    \label{table:explicitvalues}
\end{table}
\subsection{Minimum and maximum eigenvalue of the expected Hessian at true signal}
From the proof of \cref{lem:local_geometry}, we know that
\[ \mu = \lambda_{\min}(H(\beta^\star)) = \lambda_{\min}\left( \mathbb{E}\left[ (f'(X^\top\beta^\star))^2 XX^\top \right] \right),\quad \mu_{1} = \lambda_{\max}(H(\beta^\star)) = \lambda_{\max}\left( \mathbb{E}\left[ (f'(X^\top\beta^\star))^2 XX^\top \right] \right)  \]
Let $z^* = X^\top\beta^\star$. Since $X \sim \mathcal{N}(0, \mathbf{I}_d)$, then $z^*\sim \mathcal{N}(0, 1)$ because $\|\beta^\star\|=1$. Due to the rotational symmetry of the Gaussian distribution, the matrix $H(\beta^\star)$ has only two distinct eigenvalues:
\begin{enumerate}
    \item $\lambda_{\parallel}$: Associated with the eigenvector aligned with $\beta^\star$.
    \item $\lambda_{\perp}$: Associated with eigenvectors orthogonal to $\beta^\star$ (with multiplicity $d-1$).
\end{enumerate}
These are computed as:
\begin{align*}
    \lambda_{\perp} &= \mathbb{E}_{z^*}\left[ (f'(z^*))^2 \right] \\
    \lambda_{\parallel} &= \mathbb{E}_{z^*}\left[ (f'(z^*))^2 (z^*)^2 \right]
\end{align*}
We compute these explicitly for each link function. Thus, 
$\mu=\min\{\lambda_{\perp},\lambda_{\parallel}\}$, and 
$\mu_{1}=\max\{\lambda_{\perp},\lambda_{\parallel}\}$.

\subsection{Derivation of $C_{\mathrm{lip}}(R)$ and $R$}

For $R>0$, define
\[
M_1(R)
:=
\sup_{\beta\in\Ball{\bstar}{R}}
\left(\E\abs{f'(X^\top\beta)}^4\right)^{1/4},
\qquad
M_2(R)
:=
\sup_{\beta\in\Ball{\bstar}{R}}
\left(\E\abs{f''(X^\top\beta)}^4\right)^{1/4},
\]
and set
\[
C_{\mathrm{lip}}(R):=M_1(R)M_2(R).
\]

Since $X\sim\mathcal N(0,\Id)$, the random variable $X^\top\beta$ has the
same distribution as $\norm{\beta}_2 Z$, where $Z\sim\mathcal N(0,1)$.
Moreover, $\norm{\bstar}_2=1$ and $\norm{\beta-\bstar}_2\le R$ imply
\[
\norm{\beta}_2\in I_R
:=
\left[\max\{0,1-R\},\,1+R\right].
\]
Consequently,
\[
M_1(R)
=
\sup_{s\in I_R}
\left(\E\abs{f'(sZ)}^4\right)^{1/4},
\qquad
M_2(R)
=
\sup_{s\in I_R}
\left(\E\abs{f''(sZ)}^4\right)^{1/4}.
\]
Thus,
\begin{equation}\label{eq:Clip-one-dimensional}
C_{\mathrm{lip}}(R)
=
\left(
\sup_{s\in I_R}\E\abs{f'(sZ)}^4
\right)^{1/4}
\left(
\sup_{s\in I_R}\E\abs{f''(sZ)}^4
\right)^{1/4}.
\end{equation}
In particular, $C_{\mathrm{lip}}(R)$ is determined entirely by
one-dimensional Gaussian integrals and is independent of the ambient
dimension.

Recall that
\[
\mu
:=
\min\left\{
\E[f'(Z)^2],
\E[Z^2f'(Z)^2]
\right\}.
\]
The proof of \cref{lem:local_geometry} gives
\[
\norm{H(\beta)-H(\bstar)}_{\mathrm{op}}
\le
3\sqrt{15}\,C_{\mathrm{lip}}(R)
\norm{\beta-\bstar}_2
\]
for every $\beta\in\Ball{\bstar}{R}$. Hence it is sufficient to choose
$R$ such that
\begin{equation}\label{eq:radius-condition-computation}
6\sqrt{15}\,R C_{\mathrm{lip}}(R)\le\mu.
\end{equation}
Indeed, this ensures
$\norm{H(\beta)-H(\bstar)}_{\mathrm{op}}\le\mu/2$ throughout the ball.

For each link function, we therefore compute $C_{\mathrm{lip}}(R)$ from
\eqref{eq:Clip-one-dimensional} and select a radius satisfying
\eqref{eq:radius-condition-computation}. The largest radius certified by
this argument is
\[
R_{\max}
:=
\sup\left\{
r>0:
6\sqrt{15}\,r C_{\mathrm{lip}}(r)\le\mu
\right\}.
\]
This is a one-dimensional optimization problem and can be solved
numerically.

When $f'$ and $f''$ are globally bounded by $B_1$ and $B_2$,
respectively, the simpler estimate
$C_{\mathrm{lip}}(R)\le B_1B_2$ gives the conservative radius
\[
R\le\frac{\mu}{6\sqrt{15}\,B_1B_2}.
\]
The direct Gaussian-moment calculation above generally gives a sharper
radius and applies equally to links with unbounded derivatives, provided
the required fourth moments are finite.

\subsection{Derivation of $c$, $\phi_1$, and $\phi_2$}

The spectral parameter is
\[
c
:=
\operatorname{ESC}(\bstar;f)
=
\E\!\left[f'(Z)^2+f(Z)f''(Z)\right],
\qquad
Z\sim\mathcal N(0,1),
\]
where we used $\norm{\bstar}_2=1$.

The gradient-moment parameters in
\cref{theorem: li-time-grads-descent} are
\[
\phi_1
:=
\sup_{\beta\in\Ball{\bstar}{R}}
\left(\E\abs{f'(X^\top\beta)}^{16}\right)^{1/4},
\qquad
\phi_2
:=
\sup_{\beta\in\Ball{\bstar}{R}}
\left(\E\abs{f'(X^\top\beta)}^4\right)^{1/2}.
\]
As above, $X^\top\beta$ has the same distribution as
$\norm{\beta}_2 Z$, with $\norm{\beta}_2\in I_R$. Therefore,
\begin{equation}\label{eq:phi-one-dimensional}
\phi_1
=
\sup_{s\in I_R}
\left(\E\abs{f'(sZ)}^{16}\right)^{1/4},
\qquad
\phi_2
=
\sup_{s\in I_R}
\left(\E\abs{f'(sZ)}^4\right)^{1/2}.
\end{equation}
Once the basin radius $R$ has been fixed, both constants are obtained
from one-dimensional Gaussian integrals and an optimization over the
compact interval $I_R$. If $\abs{f'(z)}\le B$ for every $z$, then the simple bounds $\phi_1\le B^4$, $\phi_2\le B^2$ follow immediately. These bounds are valid but may be loose; the values
reported in \cref{table:explicitvalues} are obtained by evaluating
\eqref{eq:phi-one-dimensional} directly.
For phase retrieval, $f(z)=z^2$ and $f'(z)=2z$. Since Gaussian moments
increase with $s\ge0$, the suprema occur at $s=1+R$, giving $c=6,$ $
\phi_1
=
16(15!!)^{1/4}(1+R)^4,
$ $
\phi_2
=
4\sqrt{3}\,(1+R)^2.$

\end{document}